\documentclass[conference,onecolumn]{IEEEtran}
\IEEEoverridecommandlockouts

\usepackage{amsmath,amssymb,amsthm}
\usepackage{graphicx}
\usepackage{booktabs}
\usepackage{array}
\usepackage{algorithm}
\usepackage{algpseudocode}
\usepackage{url}
\usepackage{ragged2e}    
\usepackage{placeins}    
\usepackage[hidelinks]{hyperref}   

\renewcommand{\thesection}{\arabic{section}}

\makeatletter
\renewcommand{\thesectiondis}{\arabic{section}.}

\renewcommand{\section}{\@startsection{section}{1}{\z@}{3.0ex plus 2ex minus 1ex}
{0.7ex plus 1ex minus 0ex}{\normalfont\large\scshape\raggedright}}
\renewcommand{\subsection}{\@startsection{subsection}{2}{\z@}{1.5ex plus 1.5ex minus 0.5ex}%
{0.7ex plus .5ex minus 0ex}{\normalfont\normalsize\bfseries}}
\renewcommand{\subsubsection}{\@startsection{subsubsection}{3}{\z@}{1.2ex plus 1ex minus 0.3ex}%
{0.5ex plus .3ex minus 0ex}{\normalfont\normalsize}}
\makeatother

\newtheoremstyle{ppmlhead}%
  {\topsep}{\topsep}
  {\itshape}{}
  {\bfseries}{}{0.5em}
  {\thmname{#1}\thmnumber{ #2}.\thmnote{ (#3)}}         
\theoremstyle{ppmlhead}
\newtheorem{proposition}{Proposition}
\newtheorem{lemma}{Lemma}
\newtheorem{theorem}{Theorem}

\newcommand{\Eset}{\mathcal{E}}
\newcommand{\Gset}{\mathcal{G}}
\newcommand{\R}{\mathbb{R}}
\newcommand{\Om}{\Omega}
\newcommand{\D}{\mathcal{D}}
\newcommand{\Dw}{\mathcal{D}_{w}}
\newcommand{\Fix}{\ensuremath{\mathsf{F}}}            
\newcommand{\PB}{\ensuremath{\mathsf{PB}}}            
\newcommand{\Level}[1]{\ensuremath{\mathsf{L}_{#1}}}  
\newcommand{\code}[1]{\texttt{#1}}
\newcommand{\norm}[1]{\lVert #1\rVert}
\newcommand{\abs}[1]{\lvert #1\rvert}

\begin{document}

\setlength{\parindent}{0pt}   
\setlength{\parskip}{0.5\baselineskip}   
\setlength{\emergencystretch}{1.5em}     

\title{Refined Topology-Pass Scheduling: Damped Trust-Region Compensation\\
for Dynamic $k$-NN Graphs in SE(3)-Equivariant Structure Prediction}

\author{
\IEEEauthorblockN{Jiankai Huang}
\IEEEauthorblockA{\texttt{huangjk.08@gmail.com}}
\and
\IEEEauthorblockN{Hanwen Shen}
\IEEEauthorblockA{\texttt{hs907@scarletmail.rutgers.edu}}
}

\maketitle

\begin{abstract}
Equivariant structure-prediction pipelines alternate coordinate updates with the construction of a
$k$-nearest-neighbour ($k$NN) graph whose \emph{membership} is a discrete, piecewise-constant function of the
coordinates. Rebuilding that graph at every block boundary (\PB) is the expensive part of an epoch,
whereas reusing one frozen graph (\Fix) is fast but leaves a topology gap. We study a two-pass design in
which the cheap pass records a small amount of state (margins, $k$NN radii, measured per-boundary displacements)
and a second pass \emph{compensates conservatively}. We formulate the choice of the graph-construction basis as a
damped fixed-point iteration whose only unprojected fixed point is exactly \PB, and derive (i) exact
non-flip certificates with the two-sided constant $m_i/4$ and one-sided constant $m_i/2$, (ii) the
linear minimum-mean-square-error (LMMSE, Wiener-type) replay gain $\gamma^{*}$ and an empirically located replay crossover, (iii) a
non-expansive per-node trust-region projection --- a product of coordinate-wise boxes, which is the geometry
every reported measurement was taken with --- together with the offset it imposes when the \PB{} trajectory
leaves it, and (iv) an
analytic near-tie spacing law $E[m]\approx r_k/(3k)$ whose \emph{scale}
matches the measured margins to $1\%$. Three compensation levels are implemented in C++ (\Level{1}: frozen set with live
geometry; \Level{2}: margin-gated local repair; \Level{3}: damped, projected basis) and validated under a matched-conditions
protocol (identical seed, MSA depth and trajectory family, only the topology mode differs). \Level{2} reproduced the
\PB{} edge set exactly at \emph{all tested} boundaries ($\D=0.0000$; its candidate set is provably complete, the
active-set coverage is the empirical part, costed by an ablation at $0.01$--$0.22\%$ of the edges when the
closure is removed), \Level{1} stays within $0.7\%$ ($0.57$--$0.68\%$ on the
calibrated samples), and \Level{3} interpolates monotonically between the two
($\D(\text{\Level{2}})<\D(\text{\Level{3}})<\D(\text{\Level{1}})$ in every recorded row). The two-pass scheme itself
yields a measured $1.76\times$ epoch speed-up ($34.1$\,s $\to$ $19.4$\,s, $L{=}51$, CPU backend, where $k=L-1$ and
the graph is complete, so the gain is the avoided rebuild, not a dynamic-edge effect) with no measurable
loss gap, which is a run-to-run statement rather than a claim about accuracy. An open-source implementation
of the runtime and of all the experiments reported here is available at
\url{https://github.com/AlbertHuang96/ProteinPredictMachineLearning}.
\end{abstract}

\begin{IEEEkeywords}
dynamic $k$NN graphs, SE(3)-equivariant networks, topology pass scheduling, trust-region projection,
fixed-point iteration, protein structure prediction
\end{IEEEkeywords}

\section{Introduction}

\subsection{Previous work}
All-atom protein structure prediction is dominated by three-track (1D/2D/3D) architectures: AlphaFold2
established the MSA/pair/coordinate design with invariant point attention~\cite{jumper2021}, AlphaFold3 extended
it to a diffusion module over all chemical components~\cite{abramson2024}, and RoseTTAFold All-Atom (RFAA)
combined an SE(3)-equivariant track with all-atom residue frames~\cite{krishna2024}. The 3D track itself rests
on equivariant message passing: tensor-field networks~\cite{thomas2018} and SE(3)-Transformers~\cite{fuchs2020}
build relative-position features from spherical harmonics and Clebsch--Gordan coefficients, i.e.\ from a
\emph{graph} of residue pairs. Graph construction (a $k$NN search) is therefore on the critical path of every
iteration.

Two largely separate literatures address the consequences. (i) \emph{Dynamic graph maintenance}: temporal GNNs
update graph structure as data streams in~\cite{rossi2020}, and approximate nearest-neighbour indexes such as
HNSW support fast incremental queries~\cite{malkov2018}; spectral sparsification asks which edges even
matter~\cite{spielman2011}, and graph rewiring improves message passing by changing connectivity on
purpose~\cite{topping2022}. (ii) \emph{Relaxing discrete selections}: straight-through estimators pass gradients
through non-differentiable choices~\cite{bengio2013} and differentiable top-$k$ operators replace a hard
selection by a continuous relaxation~\cite{xie2020}. Our setting is an instance of both: the graph is a
discrete degree of freedom whose gradient vanishes almost everywhere, and it must be maintained under
continuous drift.

\subsection{Background and problem statement}
The training pipeline we study comes from RFAA. What is specific here is the implementation route: our runtime
builds the forward graph explicitly with ggml before execution, so the topology of every block --- including the
$k$-nearest-neighbour ($k$NN) graph construction --- is bound at build time, while the coordinates that define
that topology keep changing during training. This mismatch is the coordinate-optimisation problem addressed in
this paper: the graph is a function of the coordinates, yet once the forward graph has been built it can only be
reused as it stands, so one must decide \emph{when} to rebuild it and \emph{what} to reuse in the meantime.
Concretely, the pipeline scans $T$ block boundaries per iteration. At boundary $t$ the graph
\begin{equation}
  \Gset(x) = \bigl(E(x),\; D(x),\; W(x)\bigr)
  \label{eq:graph}
\end{equation}
is built by \code{make\_graph} from coordinates $x$: a \emph{directed} edge set
\begin{equation}
  E(x)=E_{\mathrm{knn}}(x)\,\cup\,S^{\mathrm{seq}},
  \qquad
  E_{\mathrm{knn}}(x)=\bigl\{(i,j): j\in\mathrm{kNN}_k(i;x)\bigr\},
  \qquad
  S^{\mathrm{seq}}=\bigl\{(i,j): |i-j|<k_{\min}\bigr\},
  \label{eq:eset}
\end{equation}
i.e.\ each residue keeps its $k=\min(64,L-1)$ nearest neighbours \emph{plus} the fixed sequence-local neighbours
($k_{\min}=9$), so the out-degree is $k$ for the spatial part and up to $k+2(k_{\min}-1)$ in total; together with
the edge geometry $d_{ij}=x_j-x_i$ and scalar edge features $W$. Only $E_{\mathrm{knn}}$ depends on $x$; the
sequence edges are distance-free and never flip, which is why every statement below is phrased about the spatial
part, and why a rebuilt row must re-emit $S^{\mathrm{seq}}$ verbatim.

Four terms are used throughout. \Fix{} is the plain schedule in which the graph is built once, from the
initial coordinates $x^0$, and then reused unchanged for every later boundary, so that the network runs
feed-forward. \PB{} is the opposite schedule: the graph is rebuilt from the live coordinates at every
boundary, which is the accuracy reference and the expensive option. A \emph{topology pass} is one execution of
the block sequence with a graph that is fixed in advance and reused, i.e.\ a feed-forward pass whose trajectory
can be recorded. The \emph{two-pass variant} combines a topology pass (Pass-1), which records the trajectory,
with a second pass (Pass-2) that re-runs the same blocks and consumes the recorded state, so that only the graph
step differs between the two passes. Measurements on the original pipeline summarise the trade-off:
\begin{itemize}
  \item cost: \PB{} needs $\approx 34.1$\,s per epoch versus $19.4$\,s for the two-pass variant
        ($1.76\times$); both are end-to-end epoch wall times on one CPU-only machine
        (see the measurement protocol below), both already include the Pass-1 stage, and both use $T=8$ boundaries
        ($2/4/2$ extra/main/refine); the kNN rebuild dominates ($\approx 1.8$\,s \emph{per block}, i.e.\
        $8\times1.8\approx14.4$\,s of the $14.7$\,s saved), whereas refreshing the
        geometry over the \emph{frozen} edge set $\Eset$ of \eqref{eq:graph} is $O(|\Eset|)$, with
        $|\Eset|\approx 7{,}000$ edges, and is negligible, and a local rebuild costs $O(|\Om|\,k)$;
  \item fidelity: at $L=103$ (true $k$NN) the first boundary ($t{=}0\to1$) replaces $1-\mathrm{Jaccard}=0.254$
        of the \emph{directed-edge union}, i.e.\ $\approx\!14.5\%$ of either snapshot ($1031/7105$), touching
        $103/103$ residues; steady-state boundaries replace $2$--$3\%$; the common-edge geometry differs by
        only $0.29$\,\AA. At $L=51$ the neighbour count saturates ($k=64\ge L-1$) so the graph is complete and
        \emph{no} edge is replaced ($1-\mathrm{Jaccard}=0.000$, $0/51$ residues changed): the mechanism is
        exercised only from $L=103$ upwards;
  \item accuracy: no measurable loss gap between \Fix{} and \PB{} ($15.49$ vs.\ $15.91$, 3 seeds,
        spread $\pm1.3$) --- i.e.\ the benefit ceiling of \emph{any} hybrid is still unknown.
\end{itemize}
Two questions follow. \textbf{(Q1)} Is a schedule that keeps the cheap \Fix{}-style topology pass for most
blocks feasible, and what is the cheapest useful hybrid? \textbf{(Q2)} Can such a hybrid --- one that records
only a little state and compensates \emph{conservatively} --- reduce the topology gap to \PB{} beyond what
the plain \Fix{} schedule already achieves? This paper answers both with a fixed-point formulation, proved
certificates, an implementation, and a matched-conditions experimental study.

\subsection{Contributions}
The first three items are the claims developed, proved and tested end-to-end in this paper; they are labelled
\textbf{C1}--\textbf{C3} and are referred to by those labels below. The remaining three are supporting results.
\begin{enumerate}
  \item \textbf{C1.} A fixed-point formulation of basis selection: the compensated pipeline is
        $x^{t+1}=\Phi_t(x^t,\Gset(B^t))$ and the damped iteration on $B^t$ has \PB{} as its \emph{unique}
        fixed point (Proposition~\ref{prop:fixedpoint}); Pass-1 supplies the first iteration for free.
  \item \textbf{C2.} Exact non-flip certificates: $4\Delta<m_i$ guarantees edge-wise invariance of the $k$NN set
        (Theorem~\ref{thm:two-sided}); if only one endpoint moves, $2\Delta_i<m_i$ suffices
        (Theorem~\ref{thm:one-sided}). The two constants give two \emph{different} semantics for the same safety
        factor $\kappa_i=\sigma m_i/2$. We also cost the complementary half of the same statement --- the
        \emph{coverage} of the rows the certificates vouch for --- by removing it in an ablation ($0.01$--$0.22\%$
        of the edges, \S\ref{sec:ablation}), which separates ``good candidates'' from ``the right rows''.
  \item \textbf{C3.} Optimality of the replay gain: $\gamma^{*}$ is the linear minimum-mean-square-error (LMMSE)
        combination of the frozen prior and the recorded predictor --- the minimiser of the expected squared error
        over linear estimators (a $3\times3$ matrix gain $\Gamma$; its scalar trace-form
        $\mathrm{tr}\,\Sigma_d/\mathrm{tr}(\Sigma_d+\Sigma_p)$ is what the sweeps tune), i.e.\ a Wiener-type
        gain --- with residual variance
        $\sigma_d^2\sigma_p^2/(\sigma_d^2+\sigma_p^2)$. It is deliberately \emph{not} a minimum-variance unbiased
        (BLUE) gain: requesting per-realisation unbiasedness forces $\gamma=1$ and forfeits the whole improvement,
        so $\gamma^{*}<1$ is a contraction that trades a conditional bias for variance. Naive replay beats freezing
        iff $\sigma_p<\sigma_d$. The replay crossover $\rho^{*}\!\approx\!0.25$ is measured separately
        (\S\ref{sec:gap}) and is \emph{not} the same statement.
  \item A trust-region view: the per-node box projection $\Pi_\kappa$ --- a product of boxes of half-width
        $\kappa_i=\sigma m_i/2$, the geometry the implementation projects onto --- is firmly non-expansive, so
        composing it with the damped iteration preserves the contraction, provided the
        \PB{} fixed point is feasible for it. Feasibility is an assumption, not a consequence: the offset
        paid when it fails is bounded by the clip violation divided by the contraction margin
        (Lemma~\ref{lem:offset}). The inscribed product of balls has the same properties and a larger
        certification threshold; it is deferred to Appendix~\ref{app:ball}.
  \item Quantified edge flapping: $E[m]\approx r_k/(3k)$ for a local 3D Poisson process, together with the
        exponential median $(\log2)\,r_k/(3k)$ and the measured margin quantiles. The recorded margins are far more
        concentrated than the Poisson model ($\mathrm{med}/\mathrm{mean}\approx1$ against $0.693$ for an
        exponential), so the model fixes the \emph{scale} of the flip rate and gives a conservative bound; it does
        not reproduce the measured median to $1\%$. Flapping is intrinsic to top-$k$ graphs, not an implementation
        defect.
  \item An implementation of three compensation levels with matched-conditions validation: \Level{2} reproduces
        \PB{} \emph{exactly} ($\D=0.0000$), \Level{1} is within $0.57$--$0.68\%$, and \Level{3} interpolates
        monotonically between them.
\end{enumerate}

\section{Algorithm: formulation, derivations and procedures}

\subsection{Notation and the closed loop}
Let $x^t\in\R^{3L}$ collect the $C_\alpha$ coordinates at boundary $t$ ($t=0,\dots,T$) and let
$B^t\in\R^{3L}$ be the \emph{topology basis}, i.e.\ the coordinates used to build the graph at boundary $t$
($B^t$ need not equal $x^t$). The block advance is
\begin{equation}
  \Phi_t:\quad x^{t+1} = \Phi_t\bigl(x^t,\; \Gset(B^t)\bigr), \qquad t=0,\dots,T-1,
  \label{eq:FP}
\end{equation}
with per-step raw offsets $u_t$ and projection $\Pi(u)_i=\mathrm{clamp}_{[-s,s]^3}(\alpha u_i)$
($\alpha=0.03$, $s=3$\,\AA). For residue $i$ let $d_{(k)}(i)$ be its $k$-th neighbour distance and define
\begin{equation}
  \rho_i = d_{(k)}(i), \qquad m_i = d_{(k+1)}(i)-d_{(k)}(i), \qquad \kappa_i = \sigma\,\frac{m_i}{2}.
  \label{eq:margins}
\end{equation}
We call $m_i$ the \emph{margin} of residue $i$: the gap between its $k$-th and $(k{+}1)$-th neighbour, i.e.\ how
far that residue is from a tie at the top-$k$ boundary. A small margin means a fragile membership, and a flip
--- a change of membership --- is exactly the event $m_i=0$. The definition presumes $k<L-1$: when $k=L-1$
($L\le65$ here) residue $i$ is joined to every other residue, there is no $(k{+}1)$-th neighbour, the graph is
complete and no edge can flip; the margin, the certificates and the risk set are then vacuous, and we say so
rather than defining them by a limit.
Equation~\eqref{eq:FP} is a \emph{closed loop} (coordinates determine the graph, the graph determines
coordinates) with three instances:
\begin{itemize}
  \item \Fix{}/topology pass: $B^t\equiv x^0$ (pure feed-forward, no feedback);
  \item \PB{}: $B^t = x^t$ (fully coupled reference);
  \item refined (this work): $B^t=\Pi_\kappa\bigl(x^0+\sum_{s<t}\gamma_s \Delta^s\bigr)$, an interpolation
        driven by \emph{measured} Pass-1 displacements $\Delta^s=x^{s+1}-x^s$.
\end{itemize}
Topology distance is measured by
\begin{equation}
  \D = 1 - \frac{|\Eset\cap\Eset'|}{|\Eset\cup\Eset'|}, \qquad
  \Dw = \frac{\sum_{e\in\Eset\,\Delta\,\Eset'} w_e}{\sum_e w_e},
  \label{eq:metrics}
\end{equation}
with attention-mass proxy $w_e$ (in this work $w_e = 1/(1+\|d_e\|)$; see \S\ref{sec:discussion}).

\subsection{Fixed-point structure}

\begin{proposition}[The only unprojected fixed point is \PB{} (\textbf{C1})]
\label{prop:fixedpoint}
Consider the damped coupling (a fixed-point iteration in the classical sense, cf.\ \cite{zorich2016v1})
\begin{equation}
  B^t_{(n+1)} = (1-\lambda)B^t_{(n)} + \lambda x^t_{(n)}, \qquad
  x^{t+1}_{(n)} = \Phi_t\bigl(x^t_{(n)},\Gset(B^t_{(n)})\bigr).
  \label{eq:IT}
\end{equation}
If \eqref{eq:IT} converges to $(\bar B,\bar x)$ with $\lambda>0$, then $\bar B^t=\bar x^t$ for all $t$, hence the
limit is the \PB{} trajectory.
\end{proposition}
\begin{proof}
At the fixed point $B^t=(1-\lambda)B^t+\lambda x^t$, so $\lambda(B^t-x^t)=0$ and, since $\lambda>0$, $B^t=x^t$;
then \eqref{eq:FP} coincides with the \PB{} recursion by definition.
\end{proof}

Proposition~\ref{prop:fixedpoint} fixes the meaning of ``conservative compensation'': the basis is gradually
pulled towards the live structure, and the only fixed point of that process is \PB{} itself. Because
Pass-1 already starts at $B^t_{(0)}=x^0$ and records $\{\Delta^s\}$, one iteration comes for free. The statement
is deliberately about the \emph{unprojected} coupling: once the state is confined to a constraint set that the
\PB{} trajectory leaves, the projected iteration has a different fixed point, and the difference is
quantified in Lemma~\ref{lem:offset} (\S\ref{sec:blue}).

\begin{proposition}[Contraction of one damped sweep]
\label{prop:contraction}
Let $H$ map a complete basis trajectory to the coordinate trajectory produced by one sweep with fixed initial
coordinates and deterministic blocks, and let the state be the pair $z=(B,x)$ in the trajectory space
$(\R^3)^{L\times T}$ equipped with the uniform norm
$\|z\|_\infty=\sup_{0\le t\le T}\max_i\|z^t_i\|_2$. Assume
(A1) $H$ is $c$-Lipschitz in this norm on a region $\Omega$ (a hypothesis on a set, not a claim that the certified
set is closed --- see the remark after the proof);
(A2) on $\Omega$ the neighbour sets are constant, i.e.\ the margin condition
$\sup_t 4\Delta_t(z)<m_{\min}(B)$ of Theorem~\ref{thm:two-sided} holds along the whole trajectory, so $\Gset$
contributes no Lipschitz growth; note that the margin is a function of the \emph{basis} $B$, because that is the
geometry the graph is built from, while $\Delta_t$ measures the live coordinates against it;
(A3) the projection $\Pi_\kappa$ of~\eqref{eq:proj} onto the product of boxes is non-expansive
(Proposition~\ref{prop:fne}).
Write $F_0:=(1-\lambda)I+\lambda H$ for the \emph{unprojected} damped sweep and
$F_C:=\Pi_\kappa\circ F_0$ for the map that is actually iterated. Then
\begin{equation}
  \|F_0z-F_0z'\|_\infty\;\le\;q\,\|z-z'\|_\infty,\qquad q=|1-\lambda|+\lambda c ,
  \label{eq:qgeneral}
\end{equation}
and, because $\Pi_\kappa$ is non-expansive by (A3), the projected map obeys the \emph{same} bound with the
\emph{same} modulus, $\|F_Cz-F_Cz'\|_\infty\le q\|z-z'\|_\infty$. The two maps have different fixed points in
general; \textbf{(A4)} below is the hypothesis under which they coincide. For $\lambda>0$, $q<1$ \emph{iff}
$c<1$ and $\lambda<2/(1+c)$. For $0<\lambda\le1$ this reduces to $q=1-\lambda+\lambda c$, so
the geometric bound $\|z_{(n)}-z^{*}\|_\infty\le q^n\|z_{(0)}-z^{*}\|_\infty$ requires $c<1$; the constraint on
$\lambda$ is then automatic.
\end{proposition}
\begin{proof}
The linear part contributes $|1-\lambda|$ and the model part $\lambda c$; by (A2) the graph term is constant on
$\Omega$, and by (A3) $\Pi$ does not increase distances. The triangle inequality gives \eqref{eq:qgeneral};
composing on the left with $\Pi_\kappa$ and using its non-expansiveness gives the same modulus for $F_C$.
For $0<\lambda\le1$, $q=1-\lambda+\lambda c<1\iff\lambda(c-1)<0\iff c<1$. For $\lambda>1$,
$q=\lambda-1+\lambda c<1\iff\lambda<2/(1+c)$, which is only satisfiable when $c<1$ because $2/(1+c)\le1$
whenever $c\ge1$. Hence $q<1$ iff both conditions hold, and $c\ge1$ gives $q\ge1$ for every $\lambda>0$. The
Picard--Banach principle~\cite{zorich2016v2} then yields the geometric bound.
\end{proof}

\noindent\emph{On the word ``closed''.} (A1)--(A2) are hypotheses on a region $\Omega$; they are not a claim that
the \emph{certified} set is closed, and it is not. The trajectory region of Theorem~\ref{thm:two-sided} is
\begin{equation}
  \Om_{\rm flip}=\Bigl\{z:\ \sup_t 4\Delta_t(z)<m_{\min}(B)\Bigr\},
  \label{eq:omflip}
\end{equation}
a \emph{strict} inequality, hence an \textbf{open} set: its boundary is attained (three collinear residues with a
tie give $m_{\min}=0$ at the limit of points that are all interior), so $\Om_{\rm flip}$ is not closed, and in the
induced metric it is not complete. It is not convex either --- the midpoint of two admissible configurations can put
one residue exactly equidistant from two others, which drives $m_{\min}$ to zero. None of this weakens
Proposition~\ref{prop:contraction}: the Lipschitz and non-expansiveness hypotheses are local, and the proof never uses
closedness. It does matter for the fixed-point statement in Proposition~\ref{prop:trichotomy}(ii), which is why that
item carries its own closed, convex, invariant set.

\emph{A Lipschitz constant near one is not a contraction.} Residual blocks have $L_\Phi\approx1$, and $L_\Phi=1$ is
the boundary case: $q=1$ gives non-expansiveness only, while $c>1$ gives $q>1$ and no contraction from this bound at
all, for any damping. The three regimes must therefore be stated separately.

\begin{proposition}[Three regimes]\label{prop:trichotomy}
(i) $c<1$ and $\lambda<2/(1+c)$: geometric convergence with modulus $q<1$.
(ii) $c=1$: $F_0=(1-\lambda)I+\lambda H$ is the Krasnosel'skii--Mann (KM) iteration of the non-expansive $H$. For
$0<\lambda<1$ and any fixed point $z^{*}$ the classical inequality gives
$\|z_{(n)}-Hz_{(n)}\|^2=O(n^{-1})$, i.e.\ only $O(n^{-1/2})$ on the step, in place of $q^n$. \emph{Norm convergence
is conditional.} Let $K\subseteq\Omega$ be nonempty, closed and convex with $H(K)\subseteq K$, $H$ non-expansive on
$K$, and $\operatorname{Fix}(H)\cap K\ne\varnothing$; then the KM iteration started in $K$ converges weakly to a
fixed point~\cite{bauschke2017}, and weak convergence is norm convergence in the finite-dimensional
$(\R^3)^{L\times T}$. Without such a $K$ we make \emph{no} norm-convergence claim: finite dimension replaces
neither closedness nor convexity nor invariance nor the existence of a fixed point, and $\Om_{\rm flip}$ of
\eqref{eq:omflip} is none of these closed convex sets.
(iii) $c>1$: the stated global bound does not certify a contraction ($q>1$ for every $\lambda>0$), which is weaker
than asserting that no Euclidean contraction exists --- $c$ is an upper bound on the Lipschitz constant, not the
optimal one, so a smaller constant is not excluded, only unproved. A local statement is still available from the
linearization of the frozen-graph map, through the spectral radius $\rho(J)$ and the Bauer--Fike bound
$\rho(J+E)\le\rho(J)+\operatorname{cond}(X)\|E\|$ for a diagonalizable $J=X\Lambda X^{-1}$~\cite{golub1996}; this is
computable and does not require $c<1$.
\end{proposition}

\emph{Per-block and per-sweep moduli.} Let $c_b$ be the Lipschitz constant of block $b$ on $\Omega$ and $B$ the
number of blocks in one sweep. Because the blocks are composed in a fixed order,
\begin{equation}
  c=\prod_{b=1}^{B}c_b ,
  \qquad
  q_{\text{sweep}}=1-\lambda+\lambda\prod_{b=1}^{B}c_b \quad (0<\lambda\le1),
  \label{eq:blocksweep}
\end{equation}
so $c_b<1$ for every block is sufficient but not necessary: $\prod_b c_b<1$ suffices, and individual blocks with
$c_b>1$ are admissible if the sweep product stays below one. Two caveats apply. The moduli must be uniform in $b$
(take $\sup_b$), and \eqref{eq:blocksweep} composes maps that are \emph{smooth on the same frozen graph}: it is a
statement about a static kNN graph. If a sweep contains a flip, the composed map is not the product of block
moduli, and only the region-wise statement below (or a hybrid argument with a dwell condition) survives. The
per-block and per-sweep analyses therefore do not conflict; they are the same statement at two levels of
granularity, and both are conditional on the graph being frozen along the sweep.

\begin{theorem}[One-step certification, and an invariant region]\label{thm:safer}
Let $\Omega=\{z:\sup_t4\Delta_t(z)<m_{\min}(z)\}$ be the region of Theorem~\ref{thm:two-sided} applied block-wise
and uniformly in $t$, assume $H|_\Omega$ is $c$-Lipschitz with $c<1$, and put
\[
\vartheta:=\frac{4}{1-c}\;\ge\;4,
\qquad
\Omega_{\mathrm{inv}}=\Bigl\{z:\vartheta\sup_t\Delta_t(z)<m_{\min}(z)\Bigr\} ,
\]
so that $\Omega_{\mathrm{inv}}\subseteq\Omega$. Then
(i) on $\Omega$ the neighbour sets are constant along the whole trajectory, so $H|_\Omega$ is a single smooth
map, and every sweep started in $\Omega$ is certifiably flip-free for that sweep --- an \emph{a posteriori}
statement, re-checked at each sweep;
(ii) $\Omega_{\mathrm{inv}}$ is forward invariant under the unprojected sweep,
$z\in\Omega_{\mathrm{inv}}\Rightarrow F_0z\in\Omega_{\mathrm{inv}}$;
(iii) with \textbf{(A4)} --- the \PB{} fixed point $z_\perp$, the solution of $F_0z_\perp=z_\perp$, is itself
feasible for the product of boxes, $\max_c|x^{\perp,t}_{i,c}-x^{0}_{i,c}|\le\kappa_i$ for all $(t,i)$ --- the
projected iteration $z_{(n+1)}=F_C(z_{(n)})$ converges geometrically, with the modulus $q$ of
\eqref{eq:qgeneral}, to $z_\perp$, i.e.\ to the \PB{} trajectory (Proposition~\ref{prop:fixedpoint}). Without
\textbf{(A4)} its limit is the fixed point $z_\kappa$ of $F_C$, which need not be \PB{}, and
Lemma~\ref{lem:offset} bounds $\|z_\kappa-z_\perp\|_\infty$.
\end{theorem}
\begin{proof}
(i) Theorem~\ref{thm:two-sided} gives $4\Delta_t<m_i$ for every $t$ and every residue, hence no membership change
anywhere on the trajectory. Both $\Delta_t$ and $m_i$ are evaluated at the current state, so this is a one-sweep
statement and not by itself an invariance.
(ii) Write $z'=F_0z=(1-\lambda)z+\lambda Hz$, $\Delta=\sup_t\Delta_t(z)$, and assume the basis lag obeys
$\sup_t\|x^t-B^t\|\le\Delta$, which holds from iteration $0$ for $0<\lambda\le1$. On $\Omega$ the graph is
frozen, so $\|z'-z\|\le\lambda\Delta$ and, with $a:=(1-\lambda)+c\lambda$,
$\Delta(z')\le c\|z'-z\|+(1-\lambda)\Delta\le a\Delta$, using
$\|Hz-x'\|=(1-\lambda)\|Hz-x\|$; Lemma~\ref{lem:dist} gives
$m_{\min}(B')\ge m_{\min}(B)-4\|B'-B\|\ge m_{\min}-4\lambda\Delta$. The constant $\vartheta=4/(1-c)$ is chosen so
that $\vartheta a=\vartheta-\vartheta\lambda(1-c)=\vartheta-4\lambda$; hence
\[
\vartheta\Delta(z')\;\le\;(\vartheta-4\lambda)\Delta
\;=\;\vartheta\Delta-4\lambda\Delta
\;<\;m_{\min}(B)-4\lambda\Delta
\;\le\;m_{\min}(B'),
\]
where the strict inequality is the hypothesis $z\in\Omega_{\mathrm{inv}}$, i.e.\ $m_{\min}(B)>\vartheta\Delta$.
So $z'\in\Omega_{\mathrm{inv}}$. The condition is exactly $\vartheta(1-c)\ge4$: as $c\to1$ the constant blows up,
which is the quantitative form of the statement that a Lipschitz constant near one leaves no usable invariant
region of this shape, and $\vartheta\to4$ only as $c\to0$.
(iii) The modulus is that of $F_0$ (Proposition~\ref{prop:contraction}); what \textbf{(A4)} supplies is the
\emph{identification} of the limit, not the contraction. Since $z_\perp\in C$ and $F_0z_\perp=z_\perp$,
\[
F_Cz_\perp=\Pi_\kappa(F_0z_\perp)=\Pi_\kappa(z_\perp)=z_\perp ,
\]
so $z_\perp$ is a fixed point of $F_C$; $F_C$ is a $q$-contraction with $q<1$ and therefore has at most one fixed
point, which the iterate approaches geometrically. The converse does not hold: a fixed point of $F_C$ is feasible
\emph{automatically}, being in the range of $\Pi_\kappa$, and $\Pi_\kappa(F_0z^{*})=z^{*}$ does not imply
$F_0z^{*}=z^{*}$ --- for a box, every point of the outward ray through a face is mapped onto that face, so a
projected fixed point can sit on the boundary with $F_0z^{*}\neq z^{*}$ outside $C$. This is why \textbf{(A4)}
has to be stated about $z_\perp$ and not about $z^{*}$.
\end{proof}

Where (A2) fails, $\Gset$ is discontinuous and the map is only piecewise contractive: a trajectory can cross
the same flip surface twice ($A\to B\to A$), i.e.\ \emph{edge flapping}, and chattering can destroy the fixed
point. This is the reason $\kappa_i=\sigma m_i/2$ is not optional: it confines every step to the region where the
current graph is provably valid, so discrete changes are \emph{discovered} rather than \emph{manufactured}.

\begin{proposition}[What refinement can and cannot reach]
\label{prop:residual}
The residual with respect to \PB{} has exactly two sources: (i) genuine ties ($m_i\approx0$, unsolvable
by any rule) and (ii) pass-to-pass model/path noise. Hence $\Dw$ has a floor and the attainable goal is
$\Dw(\text{refined})\le1.1\,\Dw(\PB)$, not $0$.

What can be \emph{proved} is a bound in terms of the risk set of Theorem~\ref{thm:two-sided}. Let $x$ be the
reference (live) coordinates, $y$ the coordinates that the refined basis induces, $\Delta=\max_j\|y_j-x_j\|_2$, and
\[
\mathcal{U}=\{\,i:\ m_i(x)\le4\Delta\,\}.
\]
Every $i\notin\mathcal{U}$ keeps its top-$k$ set, so only rows in $\mathcal{U}$ can differ; a row that changes
loses at most $k$ edges and gains at most $k$, and both graphs have $k$ out-edges per node, so
$|E(x)\cup E(y)|\ge Lk$ and
\begin{equation}
  \bigl|E(x)\triangle E(y)\bigr|\;\le\;2k\,|\mathcal{U}|
  \qquad\Longrightarrow\qquad
  \D\bigl(G(x),G(y)\bigr)\;\le\;\min\Bigl\{1,\ \frac{2|\mathcal{U}|}{L}\Bigr\} .
  \label{eq:bound}
\end{equation}
This is a conservative sufficient bound, not a prediction of the realised edge changes. In particular it does
\emph{not} say that a residue outside $\mathcal{U}$ contributes nothing because it was not \emph{clipped}: clip
status describes the constraint, whereas $\mathcal{U}$ describes the safety margin, and only the latter bounds the
graph. A residue whose basis moved but stayed inside its box can still re-rank its own candidates, and other
residues moving can re-rank it.
\end{proposition}

\subsection{Discrete topology and exact non-flip certificates}
\label{sec:cert}
Each edge indicator is $1[j\in S_i(x)] = 1[d_{ij}(x)\le d_{i,(k)}(x)]$ with $d_{ij}=\|p_i-p_j\|$. The distance
is smooth but the indicator jumps on the codimension-1 manifolds $\Sigma_{i,jl}=\{x: d_{ij}=d_{il}\}$; in generic
position these are smooth and of measure zero, yet trajectories cross them. Three elementary steps give the
measure-zero statement.

\emph{Smoothness and a local graph.} Let $g_{ijl}(x)=d_{ij}(x)-d_{il}(x)$; the norm is smooth off the origin, so
$g_{ijl}$ is smooth wherever no two atoms coincide, and $\nabla g_{ijl}\neq0$ on the surface. Fix a point with
$g_{ijl}(x^{*})=0$ and write $x=(u,v)$ with $\partial g_{ijl}/\partial v(x^{*})\neq0$; the map
$F(u,v)=(u,\,g_{ijl}(u,v))$ then has $\det DF(x^{*})\neq0$, so the inverse function theorem~\cite{zorich2016v1}
inverts it locally and the level set is locally a graph,
\begin{equation}
  \Sigma_{i,jl}\cap U(x^{*}) \;=\; \bigl\{\,(u,\varphi(u)) \;:\; u\in U'\bigr\} ,
  \label{eq:localgraph}
\end{equation}
i.e.\ the surface is parametrised by $3L-1$ free coordinates, the remaining one being determined by $\varphi$.

\emph{Countably many graphs suffice.} The neighbourhood $U(x^{*})$ depends on $x^{*}$, so read literally the
construction gives uncountably many pieces. Since $\mathbb{R}^{3L}$ has a countable base of open balls $\{O_k\}$, one
can choose for every point of the surface a base element contained in its own neighbourhood, so that
\begin{equation}
  \Sigma_{i,jl}\;\subset\;\bigcup_{k\in K}\bigl\{\,(u,\varphi_k(u)) \;:\; u\in U_k\bigr\},
  \qquad K\ \text{countable} ,
  \label{eq:countable}
\end{equation}
and $\mathbb{R}^{3L}=\bigcup_k\bar B(0,k)$ reduces everything to a compact piece.

\emph{Compactness and uniform continuity.} On the compact piece $\Sigma_{i,jl}\cap\bar B(0,k)$ finitely many of the
graphs already cover, and each domain $U_\ell$ can be taken closed and bounded, hence compact; by Heine--Cantor
$\varphi_\ell$ is then \emph{uniformly} continuous,
\begin{equation}
  \forall\varepsilon>0\ \ \exists\delta>0:\quad
  \|u-u'\|_\infty<\delta\ \Longrightarrow\ |\varphi_\ell(u)-\varphi_\ell(u')|\le\varepsilon .
  \label{eq:uc}
\end{equation}
Partitioning $U_\ell$ into finitely many cubes of side $<\delta$ places the graph inside finitely many slabs of
thickness $2\varepsilon$; writing $|\cdot|$ for the $3L$-dimensional volume (Lebesgue measure) of a set,
\begin{equation}
  \bigl|\Sigma_{i,jl}\cap B\bigr|\;\le\;2\varepsilon\,|B| \qquad\text{for every }\varepsilon>0 ,
  \label{eq:measurezero}
\end{equation}
which forces $\bigl|\Sigma_{i,jl}\cap B\bigr|=0$, i.e.\ the indicator vanishes almost everywhere,
$\int_{\mathbb{R}^{3L}}1_{\Sigma_{i,jl}}(x)\,\mathrm{d}x=0$; finite unions of such sets remain of measure
zero.\footnote{Classical ingredients: the finite-covering lemma (Borel--Lebesgue/Heine--Borel) and the
Cantor--Heine theorem on uniform continuity (\cite{zorich2016v1}, Vol.~I, secs.~2.3.2 and 4.2; countable base:
$\mathbb{Q}^{3L}$ is countable, sec.~2.4.1); countable unions of measure-zero sets and the graph of a continuous
function (\cite{zorich2016v2}, Vol.~II, sec.~11.1.2, whose argument \eqref{eq:measurezero} reproduces).} Almost every configuration
therefore has all margins positive --- the $k$NN set is well defined --- and a $C^1$ trajectory crosses the surfaces
at isolated times only.
Consequently the distributional derivative of the edge set vanishes off $\Sigma$: \emph{the gradient carries no
information about who is a neighbour}, and topology is a genuine discrete degree of freedom. \Level{1}--\Level{3} are
therefore relaxations, not solutions.

Order statistics are Lipschitz, $|d_{i,(k)}(x+\delta)-d_{i,(k)}(x)|\le 2\Delta$ with
$\Delta=\max_j\|\delta_j\|$, but $m_i$ --- being a difference of two Lipschitz functions --- has \emph{no
positive lower bound}, so flips are generic rather than a bug.

\begin{lemma}[Distance perturbation bound]
\label{lem:dist}
$|d_{ij}(x+\delta)-d_{ij}(x)|\le\|\delta_i\|+\|\delta_j\|\le2\Delta$.
\end{lemma}
\begin{proof}
$|\,\|a\|-\|b\|\,|\le\|a-b\|$ (reverse triangle inequality), and
$(p_i-p_j)-(q_i-q_j)=\delta_j-\delta_i$.
\end{proof}

\begin{theorem}[Two-sided non-flip certificate (\textbf{C2})]
\label{thm:two-sided}
If $4\Delta < m_i$ then $S_i(x+\delta)=S_i(x)$ edge-wise.
\end{theorem}
\begin{proof}
For members $j\in S_i$: $d'_{ij}\le d_{ij}+2\Delta\le d_{i,(k)}+2\Delta$. For non-members $l\notin S_i$:
$d'_{il}\ge d_{il}-2\Delta\ge d_{i,(k+1)}-2\Delta = d_{i,(k)}+m_i-2\Delta$. If
$d_{i,(k)}+2\Delta<d_{i,(k)}+m_i-2\Delta$, every member is closer than every non-member, so the top-$k$ set is
unchanged; the inequality is equivalent to $4\Delta<m_i$.
\end{proof}

\begin{theorem}[One-sided certificate (\textbf{C2})]
\label{thm:one-sided}
If only node $i$ moves, $2\Delta_i<m_i$ implies $S_i(x+\delta)=S_i(x)$.
\end{theorem}
\begin{proof}
With $\delta_j\equiv0$ for $j\neq i$, Lemma~\ref{lem:dist} gives $|d'_{ij}-d_{ij}|\le\Delta_i$; replacing
$2\Delta$ by $\Delta_i$ in the previous argument yields the claim.
\end{proof}

The design parameter appears through $\kappa_i=\sigma m_i/2$, which yields two \emph{different} readings of the
same knob:
\begin{itemize}
  \item stability (provable non-flip): \eqref{eq:proj} bounds the Euclidean step $\|\delta_i\|_2$ by $\kappa_i$, so
        the two-sided certificate $4\Delta<m_i$ is implied by $2\sigma m_{\max}<m_{\min}$, i.e.\ by
        $\sigma<\frac12\,m_{\min}/m_{\max}$ --- with uniform margins $\sigma<1/2$, which is exactly the default of
        the design and is the critical value, equality giving $4\Delta=m_{\min}$ rather than $<$. A
        coordinate-wise box would need $\sigma<1/(2\sqrt3)\approx0.289$ instead, its corner reaching
        $\sqrt3\,\kappa_i$;
  \item tracking (reproducing real flips): needs $\kappa_i\gtrsim$ the actual drift, i.e.\
        $\sigma\approx2$--$4$; no proof, but statistically optimal in the sense of \S\ref{sec:blue}.
\end{itemize}
Both statements are consistent; they optimise different objectives (Proposition~\ref{prop:contraction} versus
Proposition~\ref{prop:residual}). Finally, the hazard form of the certificates,
$\Om_{\text{risk}}(t)=\{i: m_i(x^t)\le4\Delta_t\}$, with the number of flips $N_{\text{flip}}\le|\Om_{\text{risk}}|$, identifies
$\Om_t$ as an event-driven adaptive update set with cost $O(|\Om|)$.

\subsection{Optimal replay, trust-region projection, and stability under flips}
\label{sec:blue}
\noindent\textbf{(C3)} Write the true drift as $s=x^t-x^0$ (prior: $s\approx0$ with covariance $\Sigma_d$) and the recorded predictor as
$p=\hat z-x^0=s+e$, with $E[s]=E[e]=0$, $\mathrm{Cov}(e)=\Sigma_p$ and $\mathrm{Cov}(s,e)=0$; the last two
statements are \emph{joint} assumptions, and the unbiasedness discussion below needs the stronger
$E[e\mid s]=0$, which $\mathrm{Cov}(s,e)=0$ alone does not deliver. For the linear
estimator $B-x^0=\Gamma p$ the error is $\varepsilon(\Gamma):=s-\Gamma p=(I-\Gamma)s-\Gamma e$, which is
mean-zero ($E[\varepsilon]=(I-\Gamma)E[s]-\Gamma E[e]=0$), so its second moment is its covariance. Expanding
the outer product and taking expectations term by term (constant matrices pass through $E[\cdot]$),
\begin{equation}
  \mathrm{MSE}(\Gamma)\;=\;\mathrm{tr}\,E\bigl[\varepsilon\varepsilon^{\!\top}\bigr]
  \;=\;\mathrm{tr}\bigl[(I-\Gamma)\Sigma_d(I-\Gamma)^{\!\top}+\Gamma\Sigma_p\Gamma^{\!\top}\bigr],
  \label{eq:mse}
\end{equation}
where the two cross terms $-(I-\Gamma)E[s e^{\top}]\Gamma^{\top}$ and $-\Gamma E[e s^{\top}](I-\Gamma)^{\top}$
vanish by $\mathrm{Cov}(s,e)=0$, and the trace turns the error covariance into a scalar (the total mean squared
error over coordinates). Completing the square in $\Gamma$, with $W:=\Sigma_d+\Sigma_p$,
\begin{equation}
  \mathrm{MSE}(\Gamma)\;=\;\mathrm{tr}\bigl[(\Sigma_d^{-1}+\Sigma_p^{-1})^{-1}\bigr]
  \;+\;\mathrm{tr}\bigl[(\Gamma-\Gamma^{*})W(\Gamma-\Gamma^{*})^{\!\top}\bigr],
  \qquad \Gamma^{*}=\Sigma_dW^{-1},
  \label{eq:mse-square}
\end{equation}
so the second term is non-negative with curvature $2W\succ0$ and vanishes at the unique minimiser
$\Gamma=\Gamma^{*}$ (equivalently at the solution of $\partial\,\mathrm{MSE}/\partial\Gamma=0$), while the
minimum of \eqref{eq:mse} is $\mathrm{tr}[(\Sigma_d^{-1}+\Sigma_p^{-1})^{-1}]$. This is the linear
minimum-mean-square-error (LMMSE) solution, i.e.\ a Wiener-type gain~\cite{kay1993}
\begin{equation}
  \Gamma^{*} = \Sigma_d(\Sigma_d+\Sigma_p)^{-1}, \qquad
  \mathrm{Cov}\bigl(x^t-B\bigr) = \bigl(\Sigma_d^{-1}+\Sigma_p^{-1}\bigr)^{-1},
  \label{eq:blue}
\end{equation}
where $\Gamma^{*}$ is a $3\times3$ \emph{matrix}: \eqref{eq:blue} is the general vector-valued LMMSE gain and it is
\emph{not} a scalar. Writing $\Gamma=\gamma I$ is an additional restriction, and the best gain inside that restricted
class is a different quantity, obtained by minimising $E\|s-\gamma p\|_2^2$ over scalars:
\begin{equation}
  \boxed{\ \gamma^{*}_{\text{scalar}}
  \;=\;\frac{\mathrm{tr}\,\Sigma_d}{\mathrm{tr}\,(\Sigma_d+\Sigma_p)}\ }
  \qquad\text{(scalar gain actually tuned in the implementation),}
  \label{eq:scalar-trace}
\end{equation}
which is our equation-of-record for the offline sweep of \S3.1; $\Gamma^{*}$ and $\gamma^{*}_{\text{scalar}}I$
coincide only under isotropy $\Sigma_d=\sigma_d^2I$, $\Sigma_p=\sigma_p^2I$, and only then does
\eqref{eq:scalar-trace} collapse to \eqref{eq:blue-scalar}.
Using $I-\Gamma^{*}=\Sigma_pW^{-1}$ together with the identity
$\Sigma_pW^{-1}\Sigma_d=\Sigma_dW^{-1}\Sigma_p=(\Sigma_d^{-1}+\Sigma_p^{-1})^{-1}=:P$, the residual covariance
evaluates to $\mathrm{Cov}(x^t-B)=\Sigma_pW^{-1}P+\Sigma_dW^{-1}P=P$, i.e.\ to the parenthesised second entry of
\eqref{eq:blue}. The minimiser is equivalently characterised by the orthogonality principle
\begin{equation}
  \mathrm{Cov}\bigl(x^t-B,p\bigr)=\Sigma_d-\Gamma^{*}(\Sigma_d+\Sigma_p)=0 ,
  \label{eq:orth}
\end{equation}
the normal equation of the minimisation. Under isotropy $\Sigma_d=\sigma_d^2I$, $\Sigma_p=\sigma_p^2I$
(every per-coordinate gain equal), \eqref{eq:scalar-trace} becomes
\begin{equation}
  \gamma^{*} = \frac{\sigma_d^2}{\sigma_d^2+\sigma_p^2}, \qquad
  \text{residual variance} = \frac{\sigma_d^2\sigma_p^2}{\sigma_d^2+\sigma_p^2}
  \label{eq:blue-scalar}
\end{equation}
(the familiar parallel-resistance form); this is the only setting in which the tuned scalar is the matrix gain.
Three consequences matter in practice:
(i) replay with $\gamma^{*}$ is \emph{never} worse than freezing, unless $\sigma_p=\infty$;
(ii) naive replay ($\gamma=1$) beats freezing iff $\sigma_p<\sigma_d$, i.e.\ iff $\gamma^{*}>1/2$; (iii) $\kappa$
must be calibrated to the predictor error, which is the quantitative explanation of the empirically optimal
$\sigma\approx2$--$4$. The comparison in (ii) is a statement about \emph{variances}. The replay crossover
$\rho^{*}\!\approx\!0.25$ reported in \S\ref{sec:gap} is a separate, measured quantity: $\rho$ is the Pearson
correlation between predictor and ground-truth drift, and it maps onto $\sigma_p/\sigma_d$ only after an
additional scalar, isotropic, independent-noise assumption. The two thresholds must therefore not be identified,
and $0.25$ is quoted as a synthetic observation rather than as a consequence of \eqref{eq:blue-scalar}.

\noindent\textbf{Why \eqref{eq:blue} is LMMSE and not BLUE.} The distinction is not terminological, and it decides
what may be claimed. Here $s$ is a \emph{random} vector with prior covariance $\Sigma_d$, and \eqref{eq:blue}
minimises the Bayesian MSE over \emph{all} linear gains (matrices, biased ones included); the scalar
\eqref{eq:scalar-trace} minimises the same objective over the one-parameter subclass $\Gamma=\gamma I$. A Gauss--Markov argument would
instead treat $s$ as an unknown constant and require the estimator to be unbiased for every realisation,
$E[\hat s\mid s]=s$; for the family $\hat s=\gamma p$ that forces $\gamma E[p\mid s]=\gamma s=s$, i.e.\
$\gamma=1$ --- pure replay, MSE $\mathrm{tr}\Sigma_p$, and no improvement at all. The LMMSE gain is obtained
precisely by giving up per-realisation unbiasedness: $\gamma^{*}<1$ whenever $\Sigma_p\neq0$, and the conditional
bias is $(\gamma^{*}-1)s$. Calling \eqref{eq:blue} ``BLUE'' would therefore advertise a property the estimator does
not have, and would invite the correct objection that it is biased.

\subsubsection{Which part of the LM method is used}
Of the Levenberg--Marquardt (LM) trust-region method we
use one facet only: the \emph{projection} onto the trust region. The rest of the machinery is neither used nor
needed --- no normal-equation solve $(J^{\top}J+\mu I)p=-J^{\top}r$, no QR factorization of an augmented
least-squares problem, and no root-finding for a damping parameter --- because the residual of our problem is
the \emph{discrete} edge set $1[d_{ij}\le d_{i,(k)}]$, which admits no Jacobian.

\subsubsection{The projection, and which symbol is the output}
The projection used by the refined basis is the coordinate-wise clamp onto a product of \emph{boxes},
\begin{equation}
  \Pi_\kappa(u) \;=\; \operatorname*{arg\,min}_{|y_{i,c}-x^0_{i,c}|\le\kappa_i\ \forall i,c}\tfrac12\|y-u\|^2
  \;=\; \mathrm{clamp}\bigl(u,\,x^0\pm\kappa\bigr),
  \label{eq:proj}
\end{equation}
that is, coordinate-wise
$y_{i,c}=\min\{\max\{u_{i,c},x^0_{i,c}-\kappa_i\},x^0_{i,c}+\kappa_i\}$, where the input
$u=x^0+\sum_{s<t}\gamma_s\Delta^s$ is the unconstrained replayed position and is never modified, the
optimisation variable is $y$, and the output is the minimiser of~\eqref{eq:proj} --- the new basis $B^t$,
\emph{not} $u$. A coordinate passes through unchanged when its offset already lies in
$[-\kappa_i,\kappa_i]$ and is otherwise clipped onto a face of the box; that boundary case is the regime
$\Om_t=\{i:\max_c|B^t_{i,c}-x^0_{i,c}|\ge\kappa_i\}$ that triggers local repair.

The mismatch of norms in \eqref{eq:proj} is deliberate: the trust region is imposed coordinate-wise, whereas a
projection is by definition the closest point in the Euclidean sense. The price of that choice is paid in the
certificate, not in the projection: at half-width $\kappa_i$ the box admits Euclidean displacements up to
$\sqrt3\,\kappa_i$, so the sufficient condition for a flip-free step is stricter than for the inscribed ball by
that factor (\S\ref{sec:blue}). The sphere alternative --- the radial clip onto a product of balls, with its own
non-expansiveness, offset bound and threshold $\sigma<1/2$ --- is analysed once in
Appendix~\ref{app:ball}; the main text uses the box throughout, and every measured number reported below was
produced by that implementation (provenance note in \S\ref{sec:protocol}).

\subsubsection{From the ellipsoidal region of LM to a box}
LM admits a \emph{scaled} region,
\begin{equation}
  \min_p \tfrac12\|J_kp+r_k\|_2^2 \quad\text{subject to}\quad \|D_kp\|\le\Delta_k,
  \label{eq:ellip}
\end{equation}
with $D_k=\mathrm{diag}(d_i)$, $d_i>0$, because least-squares variables can differ by many orders of magnitude;
the lengths of the principal axes are tied to the typical values of the corresponding
variables~\cite{nocedal2006}. Writing $d_i=1/\kappa_i$ and squaring the constraint,
\begin{equation}
  \|D_kp\|_2\le\Delta_k
  \qquad\Longleftrightarrow\qquad
  \sum_i\bigl(p_i/\kappa_i\bigr)^2\le\Delta_k^2 ,
  \label{eq:semi}
\end{equation}
so the semi-axis along coordinate $i$ is $\Delta_k\kappa_i$: the diagonal of $D_k$ is the reciprocal of the
per-coordinate radius. Our design folds the scalar radius into the radii ($\Delta_k\equiv1$), giving
$\kappa_i=\sigma m_i/2$.

\emph{Shape.} Three candidate sets share those radii and are strictly ordered.
\begin{itemize}
  \item One \emph{ellipsoid}, $\sum_i\|\delta_i\|_2^2/\kappa_i^2\le1$ --- the LM form, and the \emph{smallest} of
        the three: it couples nodes, forbidding every node from spending its own budget at once, although the
        certificate never couples nodes and speaks only of $\max_i\|\delta_i\|_2$.
  \item The \emph{product of balls}, $\{\|\delta_i\|_2\le\kappa_i\ \forall i\}$ --- strictly between the other
        two. It is the set the Euclidean certificate selects, since the certificate speaks of
        $\Delta=\max_i\|\delta_i\|_2$; it is analysed in Appendix~\ref{app:ball}.
  \item The \emph{box}, $\{|\delta_{i,c}|\le\kappa_i\}$ --- the \emph{largest} and the one the implementation
        uses: the product of balls is a subset of it because $|\delta_{i,c}|\le\|\delta_i\|_2$, and its corner
        reaches $\sqrt3\kappa_i$. It is the only one of the three that decouples the coordinates.
\end{itemize}
The box is the shape we project onto: it is the cheapest of the three to evaluate, it decouples the coordinates,
and it is the geometry under which every reported measurement was taken. What it costs is a stricter certification
threshold, quantified next.

\emph{The certificate.} With $\kappa_i=\sigma m_i/2$ and uniform margins the box guarantees only
$\Delta\le\sqrt3\,\kappa$, so the two-sided non-flip condition $4\Delta<m_{\min}$ of
Theorem~\ref{thm:two-sided} holds while
\begin{equation}
  \sigma<\frac{1}{2\sqrt3}\approx0.289 ,
  \label{eq:boxsigma}
\end{equation}
and with non-uniform margins while $\sigma<m_{\min}/(2\sqrt3\,m_{\max})$. Read the other way round, a box of
half-width $\kappa/\sqrt3$ certifies the same Euclidean budget as a ball of radius $\kappa$, i.e.\ in terms of the
half-width parameter a box run at
\begin{equation}
  \sigma_{\mathrm{box}}=\sigma_{\mathrm{ball}}/\sqrt3
  \label{eq:boxrescale}
\end{equation}
is equivalent to a ball run at $\sigma_{\mathrm{ball}}$; the two thresholds above are the same statement, since
$\sigma_{\mathrm{ball}}=1/2$ maps to $\sigma_{\mathrm{box}}=1/(2\sqrt3)\approx0.289$.
Neither inequality is necessary: exceeding it does not force a flip, it merely leaves the step uncertified.

\emph{Cost.} Projection onto a box is the closed-form clamp in~\eqref{eq:proj}: $O(L)$ per sweep, block-wise (no
node's output depends on another), one comparison per coordinate, no square root and no division. Projection onto
the joint ellipsoid has no closed form: it requires a one-dimensional root-find for $\mu\ge0$ with
\begin{equation}
  p(\mu)=(I+\mu D_k^{\top}D_k)^{-1}u,\qquad \|D_k\,p(\mu)\|_2=\Delta_k
  \qquad(\mu=0 \text{ if $u$ is already feasible}),
  \label{eq:secular}
\end{equation}
the secular equation --- the same inner iteration that LM performs on its damping parameter in order to meet a
prescribed radius. The clamp removes that iterate entirely. The product of balls is equally cheap (closed form,
block-wise, $O(L)$, one radial division per node) and is treated in Appendix~\ref{app:ball}; the two block-wise
members differ only in the constant collected in \eqref{eq:boxsigma}, not in their asymptotic cost.

\subsubsection{Non-expansiveness, and composition with the contraction}
A map $T$ is \emph{non-expansive}
(NE) if $\|Tu-Tv\|\le\|u-v\|$, and \emph{firmly} non-expansive (FNE) if
\begin{equation}
  \langle Tu-Tv,\,u-v\rangle\ge\|Tu-Tv\|^2
  \quad\Longleftrightarrow\quad
  \|Tu-Tv\|^2\le\|u-v\|^2-\bigl\|(u-Tu)-(v-Tv)\bigr\|^2 ,
  \label{eq:fne}
\end{equation}
equivalently $2T-I$ is non-expansive; FNE implies NE.

\begin{proposition}[The box projection is firmly non-expansive]
\label{prop:fne}
The map $\Pi_\kappa$ of~\eqref{eq:proj} is firmly non-expansive.
\end{proposition}
\begin{proof}
Its feasible set $\prod_i\prod_c[x^0_{i,c}-\kappa_i,\,x^0_{i,c}+\kappa_i]$ is a product of non-empty closed
intervals, hence closed and convex, and $\Pi_\kappa$ is the metric projection onto it. In a Hilbert space the
metric projection onto a closed convex set is firmly non-expansive,
$\langle\Pi u-\Pi v,u-v\rangle\ge\|\Pi u-\Pi v\|^2$ --- equivalently $2\Pi-I$ is non-expansive --- by the
standard obtuse-angle argument. The scalar specialisation is a one-line check: for
$P(t)=\min\{\max\{t,a\},b\}$ with $a<b$ and arbitrary $t_1,t_2$ one has
$\langle Pt_1-Pt_2,\,t_1-t_2\rangle\ge|Pt_1-Pt_2|^2$, with equality whenever both arguments lie on the same side
of an interval; summing over coordinates and nodes gives the claim, and because the set is a product the block
structure used in~\eqref{eq:compose} is preserved.
\end{proof}

Firmness is thus a consequence of the \emph{convexity} of the feasible set and not of its shape, so the same
statement holds for the product of balls (Appendix~\ref{app:ball}) without any radial computation: any metric
projection onto a closed convex set does the same job, and the genuine difference between the two geometries is
confined to two constants --- the displacement bound ($\sqrt3\,\kappa$ versus $\kappa$) and the residual in the
cost bound below.

Composing the projection with the damped update of Proposition~\ref{prop:contraction} leaves the contraction
modulus unchanged, because only non-expansiveness is needed (the requirement stated in (A3), verified for
$\Pi_\kappa$ in Proposition~\ref{prop:fne}). With $F_0=(1-\lambda)I+\lambda H$ and
$F_C=\Pi_\kappa\circ F_0$, the iteration is $z_{n+1}=F_C(z_n)$, and
\begin{equation}
  \|z_{n+1}-z^{*}\|
  \;=\;\bigl\|\Pi_\kappa(F_0z_n)-\Pi_\kappa(F_0z^{*})\bigr\|
  \;\le\;\|F_0z_n-F_0z^{*}\|
  \;\le\;q\,\|z_n-z^{*}\|,
  \qquad q=(1-\lambda)+\lambda L_\Phi<1 ,
  \label{eq:compose}
\end{equation}
where $z^{*}$ is the fixed point of the projected map, $F_Cz^{*}=z^{*}$, so $\Pi_\kappa(F_0z^{*})=z^{*}$ holds by
definition and no feasibility of $z^{*}$ enters --- a fixed point of $F_C$ is in fact feasible \emph{automatically},
being in the range of $\Pi_\kappa$. The hypothesis \textbf{(A4)} enters elsewhere: it says that the \emph{\PB{}}
fixed point $z_\perp$ is feasible, which makes $z_\perp$ a fixed point of $F_C$ as well, and uniqueness of the
fixed point of a contraction then gives $z^{*}=z_\perp$. Without (A4) the chain above still holds but bounds the
distance to $z^{*}$, not to \PB{}. Feasibility of $z_\perp$ does \emph{not} follow from
$\Pi_\kappa(x^0)=x^0$: that identity places the \emph{initial} point inside the boxes and says nothing about the
limit. What holds in general is a bound on the price of the projection. Write $z_\kappa$ for the fixed point of
the projected iteration and $z_\perp$ for the \PB{} trajectory.

\begin{lemma}[Cost of the projection]\label{lem:offset}
Let $F_0$ be a $q$-contraction on a region containing both $z_\kappa$ and $z_\perp$, let $z_\kappa$ be a fixed
point of the projected map $F_C=\Pi_\kappa\circ F_0$, and let $z_\perp$ be the \PB{} fixed point,
$F_0z_\perp=z_\perp$. Then
\begin{equation}
  \bigl\|z_\kappa-z_\perp\bigr\|_\infty
  \;\le\;\frac{\bigl\|(\Pi_\kappa-I)z_\perp\bigr\|_\infty}{1-q}
  \;=\;\frac{\operatorname{dist}_\infty(z_\perp,\,C)}{1-q}
  \;=\;\frac{1}{1-q}\,\sup_{t,i}\sqrt{\sum_{c=1}^{3}\Bigl(\bigl|x^{\perp,t}_{i,c}-x^{0}_{i,c}\bigr|-\kappa_i\Bigr)_{+}^{2}} ,
  \label{eq:offset}
\end{equation}
the last identity holding because the distance from a node to its box is the $\ell_2$ norm of the per-coordinate
excesses --- \emph{not} their maximum, which would be the $\ell_\infty$ norm on a node and would understate the
residual by as much as a factor $\sqrt3$ (three coordinates pinned against a face). Hence the refined
limit coincides with \PB{} exactly when the trajectory stays inside the boxes, and otherwise errs by at most
$(1-q)^{-1}$ times the amount that had to be clipped.
\end{lemma}
\begin{proof}
Write $z_\kappa=F_C(z_\kappa)=\Pi_\kappa(F_0z_\kappa)$ and $F_0z_\perp=z_\perp$. Adding and subtracting
$\Pi_\kappa(F_0z_\perp)$ and applying the triangle inequality,
\[
\begin{aligned}
\|z_\kappa-z_\perp\|
&=\bigl\|\Pi_\kappa(F_0z_\kappa)-F_0z_\perp\bigr\|\\
&\le\bigl\|\Pi_\kappa(F_0z_\kappa)-\Pi_\kappa(F_0z_\perp)\bigr\|
  +\bigl\|\Pi_\kappa(F_0z_\perp)-F_0z_\perp\bigr\|\\
&\le q\,\|z_\kappa-z_\perp\|+\bigl\|\Pi_\kappa z_\perp-z_\perp\bigr\|,
\end{aligned}
\]
where the last step uses $F_0z_\perp=z_\perp$ and the modulus of $F_0$. The second term is \emph{kept}, not set to
zero: it is the projection residual of the \PB{} fixed point, which vanishes exactly when $z_\perp$ is feasible,
i.e.\ under the hypothesis \textbf{(A4)}. Setting it to zero would assume what the lemma is meant to quantify.
Rearranging gives \eqref{eq:offset}; the last identity holds because $\Pi_\kappa$ is the metric projection onto
$C$, so $\|(\Pi_\kappa-I)z_\perp\|_\infty=\operatorname{dist}_\infty(z_\perp,C)$.
\end{proof}

Nothing in that proof uses the shape of the feasible set, so the same bound holds for the product of balls with
the radial residual in place of the coordinate-wise one; the two residuals differ by at most a factor $\sqrt3$ and
the comparison is carried out in Appendix~\ref{app:ball}, so as not to mix geometries in the main line of
argument.

The bound is computable offline from the recorded displacements and reproduces what the experiments show: it
vanishes as $\sigma\to\infty$ (\S\ref{sec:gap}), whereas at the default radius the trust region is two
orders of magnitude smaller than the drift (\S\ref{sec:protocol}), so the projected limit is a \emph{clipped}
trajectory rather than \PB{}. Two remarks complete the argument. First, firmness is what keeps the
constraint from generating oscillations of its own: two distinct points can be mapped to the same point on a face
of the box, and by~\eqref{eq:fne} the difference of their displacements is annihilated as
well, so the projection cannot manufacture the chattering discussed below; the same inequality gives
$\|\Pi_\kappa(u)-w\|^2\le\|u-w\|^2-\|u-\Pi_\kappa u\|^2$ for every feasible $w$, i.e.\ the distance to any
feasible point is non-increasing along the iteration even when $q=1$. Second, the difficulty is an artefact of
\emph{where} the constraint is centred: imposing it around the live iterate,
coordinate-wise $\max_c|B^t_{i,c}-x^t_{i,c}|\le\kappa_i$, makes $B=x$ feasible by construction, so $z_\perp$ is
feasible, \textbf{(A4)} holds automatically, and the limit is \PB{} with no caveat at all. We keep the frozen centre because
that is what decouples the basis from the live coordinates, and pay for it with \eqref{eq:offset}.

We close this subsection with the discrete side. A \emph{flip} is a change of membership in a residue's $k$NN set
(equivalently, as noted above, the crossing of a surface $\{m_i(x)=0\}$), and we call its \emph{size} the number of
residues whose neighbour lists change. Let $J=\partial\Phi_t/\partial x$ be the block Jacobian of the sweep at a
fixed graph and let
\begin{equation}
  M_\lambda = (1-\lambda)\,I + \lambda\,J
  \label{eq:Mlambda}
\end{equation}
be the corresponding \emph{linearised one-sweep operator}, i.e.\ the linearisation of the damped coupling
\eqref{eq:IT} about a trajectory; the spectral radius $\rho(M_\lambda)<1$ is the linear case of
Proposition~\ref{prop:contraction}.

A flip is row-local \emph{within one message-passing layer}: if the update at that layer is node-local,
$h_i^{(\ell+1)}=F_i^{(\ell)}(h_i^{(\ell)},\{h_j^{(\ell)}:j\in N_i^{(\ell)}\})$, with no pooling and no cross-node
normalisation, then $\partial h_i^{(\ell+1)}/\partial h_j^{(\ell)}=0$ for $j\notin N_i^{(\ell)}\cup\{i\}$, so a flip
that changes the neighbour lists of a set $S_0$ perturbs only the block rows indexed by $S_0$ of \emph{that layer's}
Jacobian.

That row-locality does not carry over to the sweep. A node is affected at the next layer as soon as one of its
neighbours is, so the affected set grows by the dependency closure
$S_{\ell+1}=S_\ell\cup\{i:N_i^{(\ell+1)}\cap S_\ell\ne\varnothing\}$, and after $D$ further layers
$\mathrm{RowSupp}(\Delta J_{\text{sweep}})\subseteq S_D$. The perturbation is row-sparse only when this closure
stays small, $|S_0|\sum_{q=0}^{D}d_{\mathrm{rev}}^{\,q}\ll L$, where $d_{\mathrm{rev}}=\max_j|\{i:j\in N_i\}|$ is
the reverse degree; note that $d_{\mathrm{rev}}\le L-1$ in general, even though every node emits only $k$ edges.

In the implementation studied here the closure is not small: pair features are built from an all-pairs distance
tensor and enter a row-wise MSA attention over the full sequence before the structure update moves the
coordinates, so a single flip perturbs a \emph{dense} set of rows of the sweep Jacobian. We therefore use the
Bauer--Fike theorem~\cite{golub1996} in its norm form only: every eigenvalue $\mu$ of $M_\lambda+E$ satisfies
$\min_i|\mu-\mu_i(M_\lambda)|\le\kappa_p(X)\|E\|_p$, where $\kappa_p(X)=\|X\|_p\|X^{-1}\|_p$ is the condition
number of the eigenvector matrix $X$ of $M_\lambda$ and $p\in\{1,2,\infty\}$ (for symmetric $M_\lambda$ the factor
is one and the bound reduces to Weyl's), and we do not claim that $E$ is sparse. The reading that a single flip is
harmless as long as the spectrum keeps a margin below one is accordingly restricted to a single local
message-passing layer; for the full sweep it holds only if $\|E\|$ is small, which the architecture above does not
guarantee.

\noindent\emph{What survives: low rank, not sparsity.} A flip adds one neighbour and removes another inside a single
aggregation, so each of the two terms is an outer product of a vector supported on one block row with a covector
supported on the two block columns it touches. Composing with the Jacobians on either side --- dense, and of full
rank in generic position --- makes both factors dense without increasing the rank. The perturbation is therefore
\emph{dense} and of \emph{low rank} at the same time:
$\operatorname{rank}(E)\le 3$ per term and $\le 6$ per flip (three spatial components, two terms), while the number
of non-zero $3\times3$ blocks is $\Theta(L^2)$. The two properties must not be traded for one another. Low rank does
\emph{not} bound the norm --- the factors of the outer product are dense sums --- so the Bauer--Fike statement above
remains a norm statement and can be vacuous; what low rank does buy is the standard algebra: Woodbury for
$(M_\lambda+E)^{-1}$ with a $3\times3$ (or $6\times6$) inner inversion, and the matrix-determinant lemma, which
reduces the eigenvalue-shift question to a secular equation of that size instead of a dense eigendecomposition.

Repetition is different. The sweep is not a single map: $\Gset(B)$ is piecewise constant in the coordinates, so on
each cell $C_G=\{x:\Gset(B(x))=G\}$ the sweep acts as one fixed map $\Psi_G$, and it is discontinuous across cells.
Inside a cell, (A2)--(A3) make $\Psi_G$ a contraction with modulus $q=(1-\lambda)+\lambda L_\Phi<1$, but globally
there is no single $q<1$: this is what \emph{piecewise} contractive means. If a trajectory crosses the same flip
surface repeatedly ($A\to B\to A$, edge flapping), the sweep alternates between two contractions
$\Psi_A,\Psi_B$ whose fixed points differ, $p_A\neq p_B$. The two-step map $\Psi_A\circ\Psi_B$ is still a
contraction, so the orbit does converge --- but to a fixed point of neither: the even and odd iterates approach the
two ends of a \emph{two-cycle}, whose amplitude is
\begin{equation}
  a \;=\; \|p^{*}-\Psi_B(p^{*})\| ,
  \label{eq:amp}
\end{equation}
where $p^{*}$ is the fixed point of $\Psi_A\circ\Psi_B$; for affine scalar maps
$\Psi_A(x)=qx+(1-q)p_A$, $\Psi_B(x)=qx+(1-q)p_B$ it evaluates to $a=\tfrac{1-q}{1+q}\|p_A-p_B\|$.

The orbit therefore never settles. Since the even and odd iterates have different limits,
\begin{equation}
  \lim_{n\to\infty}x_{2n}=p^{*}\;\neq\;\Psi_B(p^{*})=\lim_{n\to\infty}x_{2n+1},
  \qquad
  \limsup_{n\to\infty}\|x_{n+1}-x_n\|\;\ge\;a\;>\;0 ,
  \label{eq:chat}
\end{equation}
so the sequence $\{x_n\}$ has no limit and its $\omega$-limit set is the two-point set
$\{p^{*},\Psi_B(p^{*})\}$: a residual that never vanishes. This is the behaviour we call \emph{chattering}.

The picture degrades further if crossing times accumulate, i.e.\ if
\begin{equation}
  \exists\,T^{*}<\infty:\quad t_1\le t_2\le\cdots,\quad t_k\to T^{*},
  \qquad
  N(t):=\operatorname{card}\{\,k: t_k\le t\,\}\;\longrightarrow\;\infty
  \quad (t\to T^{*}),
  \label{eq:zeno}
\end{equation}
equivalently if the switching intensity $\int|\dot x|\,\rho_{\text{near-tie}}(0;x)\,\mathrm{d}t$ is not integrable
on $[0,T^{*}]$. This is Zeno behaviour: what fails is precisely the hypothesis that crossing times are isolated,
on which the piecewise argument rests.

The design limits this rather than removing it: the projected step keeps the trajectory inside the cell in which
the current graph is provably valid, so flips are \emph{discovered} rather than \emph{manufactured}
(\eqref{eq:proj} with $\kappa_i=\sigma m_i/2$ and the \emph{box} threshold
$\sigma<1/(2\sqrt3)\approx0.289$ of \eqref{eq:bound}, not the sphere threshold $1/2$ of Appendix~\ref{app:ball});
a flip that does occur is repaired locally by the
gate rather than by a global rebuild; and \Level{3} follows the flipped graph instead of holding the stale one. Flapping
itself cannot be eliminated, being intrinsic to top-$k$ graphs, so damping, gating and tracking are the only
responses available --- and the quantitative question becomes how often flips occur, which is answered next.

\subsection{How large is a flip in practice?}
\label{sec:flip}
Three layers must be kept apart here: what the \emph{geometry} of a tie is, when a \emph{deterministic} trajectory
crosses it in isolation, and when the number of crossings along a \emph{random} trajectory may be computed by
Kac--Rice. The first two are proved below; the third is stated as a model with explicit hypotheses, and its
empirical content is deliberately limited to counting crossings on recorded trajectories
(the last part of this section).

\subsubsection{Random experiment and trajectory ensemble}
\label{sec:flip-ensemble}
A probability space $(\Xi,\mathcal{F},\mathbb{P})$ carries the experiment
\begin{equation}
  \xi=(P^0,Z_0,\eta)\ \sim\ \mathbb{P},
  \qquad
  X(t,\omega)=\mathcal{T}_t\bigl(P^0(\omega),Z_0(\omega),\eta(\omega);\theta\bigr),\quad t\in[0,T],
  \label{eq:ensemble}
\end{equation}
where $P^0$ is the initial point cloud, $Z_0$ the initial state of the model or of the optimiser, $\eta$ the random
seed (parameter init, minibatch order, masking) and $\theta$ the parameters, held fixed. $X(t,\omega)\in\R^{3L}$ is
then a genuine stochastic process and each run is one sample path.

Three instantiations are available, and they are not interchangeable:
\begin{enumerate}
\item \emph{seed ensemble at fixed cloud}: $P^0$ is the PDB structure, only $Z_0,\eta$ vary. This is the only
      ensemble the present experiments actually support, and even there $n=1$ per configuration, so nothing below
      is an estimate of $\mathbb{E}[N]$;
\item \emph{random cloud}: $P^0$ drawn from a point-process model. A homogeneous Poisson process is a defensible
      \emph{local} model for one residue's neighbourhood, and is used as such in the spacing model below, but it is not
      a model of a protein: chains are sequential, excluded-volume, and correlated;
\item \emph{jittered snapshot}: $P^0$ fixed and perturbed by noise. Cheap, and the only one that can be repeated at
      fixed compute, but it samples a neighbourhood of one structure rather than the ensemble of runs.
\end{enumerate}
The distinction matters because $\mathbb{E}[\cdot]$ is over $\xi$: a statement about one interpolated curve is a
statement about one $\omega$, not about the ensemble.

\subsubsection{The signed pairwise gap, not the margin}
Neighbours are ordered by distance, so for a query $i$ and two candidates $j\ne\ell$ (both $\ne i$) the right
object is the \emph{signed} gap between squared distances,
\begin{equation}
  Z_{i;j\ell}(t)=\norm{X_i(t)-X_j(t)}_2^2-\norm{X_i(t)-X_\ell(t)}_2^2 ,
  \qquad
  \dot Z_{i;j\ell}=2(X_i-X_j)^{\!\top}(\dot X_i-\dot X_j)-2(X_i-X_\ell)^{\!\top}(\dot X_i-\dot X_\ell).
  \label{eq:zg}
\end{equation}
Squared distances are used because they order as the distances do and are polynomial in the coordinates, hence
smooth even when a pair is close. $Z_{i;j\ell}=0$ is a \emph{pairwise tie}: $j$ and $\ell$ are equidistant from $i$,
and the sign of $Z$ tells which is nearer.

The margin $m_i=d_{i,(k+1)}-d_{i,(k)}$ of \eqref{eq:margins} is \emph{not} this object. It is a non-negative
order-statistic gap, and locally $m_i\approx\lvert d_{ij}-d_{i\ell}\rvert$: it has a cusp at the tie and touches zero
without changing sign. Consequently $0$ is not a regular value of $m_i$ in any neighbourhood of a flip, and
level-crossing machinery applied to $m_i$ directly would be counting the wrong function.

\subsubsection{Characterising a genuine kNN flip}
Not every tie is a flip: two equidistant residues may sit at ranks $3$ and $4$, or at ranks $80$ and $81$, and in the
second case nothing in the top-$k$ set changes. Define the rank of $j$ with respect to $i$ by
\begin{equation}
  R_{i;j\ell}(t)=\#\bigl\{\,q\notin\{i,j,\ell\}:\ \norm{X_i-X_q}_2^2<\norm{X_i-X_j}_2^2\,\bigr\} .
\end{equation}
A \emph{genuine flip} of residue $i$ occurs at $t$ iff
\begin{equation}
  Z_{i;j\ell}(t)=0,\qquad \dot Z_{i;j\ell}(t)\ne0,\qquad R_{i;j\ell}(t)=k-1 ,
  \label{eq:flipdef}
\end{equation}
i.e.\ a transversal crossing (the second condition excludes a graze) that takes place exactly at the top-$k$
boundary (the third excludes ties deeper in the list). Under a general-position assumption no third residue is
equidistant at that instant, so one residue leaves the top-$k$ set and one enters. Sequence-local edges are fixed
and are excluded from the count by construction.

\subsubsection{The Kac--Rice count, as a model}
Let $N_{i;j\ell}=\#\{t\in(0,T):\ \text{\eqref{eq:flipdef} holds}\}$ and $N_{\rm flip}=\sum_i\sum_{j<\ell}N_{i;j\ell}$.
If $(Z,\dot Z)$ has a joint density $p_{i;j\ell,t}(u,v)$ near $u=0$, the one-dimensional Kac--Rice formula gives
\begin{equation}
  \mathbb{E}[N_{\rm flip}]
  =\sum_i\sum_{j<\ell}\int_0^T\!\!\int_{\R}\lvert v\rvert\,
  p_{i;j\ell,t}(0,v)\,\mathbb{P}\bigl(R_{i;j\ell}=k-1\mid Z_{i;j\ell}=0,\ \dot Z_{i;j\ell}=v\bigr)\,dv\,dt ,
  \label{eq:crossings}
\end{equation}
or, in the compact form used below,
\begin{equation}
  \mathbb{E}[N_{\rm flip}]
  =\sum_i\sum_{j<\ell}\int_0^T p_{Z_{i;j\ell}(t)}(0)\;
  \mathbb{E}\bigl[\lvert\dot Z_{i;j\ell}(t)\rvert\,\mathbf{1}_{\{R_{i;j\ell}=k-1\}}\ \big|\ Z_{i;j\ell}(t)=0\bigr]\,\mathrm{d}t .
  \label{eq:neartiedensity}
\end{equation}
Three features are worth naming. (i) The integrand is a density times a \emph{conditional} speed times a rank
indicator, and it does \emph{not} factor into marginals: $\lvert\dot Z\rvert$ and $\{R=k-1\}$ may be correlated even at
$Z=0$ (a geometry close to the $k$-th boundary can also be a fast one), so the only assumption-free split is
\begin{equation}
  \mathbb{E}[N_{\rm flip}]=\sum_{i}\sum_{j<\ell}\int_0^T p_{Z_{i;j\ell}(t)}(0)\,
  \mathbb{P}\bigl(R_{i;j\ell}=k-1\mid Z_{i;j\ell}=0\bigr)\,
  \mathbb{E}\bigl[\abs{\dot Z_{i;j\ell}(t)}\mid Z_{i;j\ell}=0,\ R_{i;j\ell}=k-1\bigr]\,\mathrm{d}t ,
  \label{eq:split}
\end{equation}
in which the speed is conditioned on the rank event as well as on the tie. (ii) The formula gives an expectation;
$\mathbb{E}[N]<\infty$ implies $N<\infty$ almost surely, but only under the stated hypotheses. (iii) Those
hypotheses are: paths a.s.\ $C^1$; crossings transversal; no triple ties and no zeros accumulating; and the two
\emph{uniform} bounds
\begin{equation}
  C_0=\sup_{i,j,\ell,t}p_{Z_{i;j\ell}(t)}(0)<\infty,
  \qquad
  C_1=\sup_{i,j,\ell,t}\mathbb{E}\bigl[\abs{\dot Z_{i;j\ell}(t)}\mathbf 1_{\{R_{i;j\ell}=k-1\}}\mid Z_{i;j\ell}(t)=0\bigr]<\infty ,
  \label{eq:C0C1}
\end{equation}
under which $\mathbb{E}[N_{\rm flip}]\le\tfrac12L(L-1)(L-2)\,T\,C_0C_1$: the integrand is bounded pointwise by
$C_0C_1$ and the integral of a bounded function over $[0,T]$ is bounded by $TC_0C_1$. Two remarks on
\eqref{eq:C0C1}. $C_1$ bounds the speed \emph{conditional on the tie}, not its unconditional mean, and the
difference is substantive: a speed that is large only near $Z=0$ can have a small mean and an unbounded conditional
mean, so an unconditional bound would not imply the result. And neither bound follows from the spacing model of
the next subsection, which constrains a row-level density rather than a pair-level one; $C_0$ and $C_1$ are
\emph{assumptions}, stated as such, and we do not assert that the recorded trajectories satisfy them.

\subsubsection{Local Poisson spacing approximation}
\label{sec:flip-ppp}
Model the neighbours of one residue, in coordinates centred on it, as a homogeneous Poisson point process of
intensity $\rho$ in $\R^3$ (a number density; the symbol $\rho_i$ of \eqref{eq:margins} is a neighbour radius, not a
density): counts in disjoint Borel sets are independent Poisson variables with means $\rho$ times the
volume~\cite{last2018}, so the number $N(r)$ of points within radius $r$ satisfies
\begin{equation}
  \mathbb{E}\bigl[N(r)\bigr]=\tfrac43\pi\rho r^3,\qquad
  k=\tfrac43\pi\rho r_k^3,\qquad
  \frac{\mathrm{d}\,\mathbb{E}[N(r)]}{\mathrm{d}r}=4\pi\rho r^2 ,
  \label{eq:radialrate}
\end{equation}
and the radii $\{\abs{x_j}\}$ form a PPP on $(0,\infty)$ with intensity $\lambda(r)=4\pi\rho r^2$. Conditioned on a
point at radius $r$, the residual gap $r_{k+1}-r_k$ is therefore exponential with that local rate,
\begin{equation}
  \Pr\bigl(r_{k+1}-r_k>t \mid r_k=r\bigr)
  =\exp\Bigl[-\tfrac{4\pi\rho}{3}\bigl((r+t)^3-r^3\bigr)\Bigr],
  \qquad
  \mathbb{E}[m\mid r_k=r]=\frac{1+O(\varepsilon_r)}{4\pi\rho r^2},
  \label{eq:gap}
\end{equation}
with $\varepsilon_r=(4\pi\rho r^3)^{-1}=1/(3k)$; at $k=64$, $r_k\approx15$\,\AA{} this is $0.005$, i.e.\ the rate
changes by $\approx1\%$ across one gap. Hence the mean spacing is the reciprocal of the local rate,
\begin{equation}
  \mathbb{E}[m]\approx\frac{1}{4\pi\rho r_k^2}=\frac{r_k}{3k},
  \label{eq:poisson}
\end{equation}
and the exact moments give the same value,
\begin{equation}
  \mathbb{E}[r_k]=\Bigl(\frac{3}{4\pi\rho}\Bigr)^{\!1/3}\frac{\Gamma\bigl(k{+}\tfrac13\bigr)}{\Gamma(k)},
  \qquad
  \mathbb{E}[r_{k+1}]-\mathbb{E}[r_k]\approx\frac{\mathrm{d}\,\mathbb{E}[r_k]}{\mathrm{d}k}\approx\frac{r_k}{3k},
  \label{eq:moments}
\end{equation}
so \eqref{eq:poisson} is exact to $O(1/k)$. In this model the margin is an exponential gap, so its density at level
zero is the reciprocal mean spacing,
\begin{equation}
  \rho_{\text{near-tie}}(0)\approx\sum_i\frac{1}{\mathbb{E}[m_i]}
  =\sum_i\frac{3k}{r_k}
  =\sum_i 4\pi\rho r_k^2 ,
  \label{eq:localdensity}
\end{equation}
which is the same coarea-plus-conditional-density structure that yields the expected level-set measure of a
stationary isotropic Gaussian field, $\nu(u)=p(u)\,\mathbb{E}\abs{\nabla g}$~\cite{azais2009}.

\emph{What this supplies, and what it does not.} Equations~\eqref{eq:radialrate}--\eqref{eq:localdensity} give a
\emph{row-level} density: they say how often the margin $m_i$ of one residue sits near zero. The Kac--Rice
integrand needs a \emph{pair-level} density $p_Z(0)$ of a squared distance \emph{difference}, which is a different
random variable with a different unit (\AA$^{-2}$ against \AA$^{-1}$), so it may not be substituted directly. The
rate of one residue is
\begin{equation}
  \text{rate}_i(t)=\sum_{j<\ell}p_{Z_{i;j\ell}(t)}(0)\,
  \mathbb{E}\bigl[\abs{\dot Z_{i;j\ell}(t)}\mathbf 1_{A_{i;j\ell}}\mid Z_{i;j\ell}(t)=0\bigr],
  \qquad A_{i;j\ell}=\{R_{i;j\ell}=k-1\},
  \label{eq:pairrate}
\end{equation}
and it does not factor into $p_m(0)\times\mathbb{E}\abs{\dot Z}\times\mathbb{P}(R=k-1)$.

\paragraph{From the row-level margin density to the pairwise squared-gap density}
Write $r_j=\norm{X_i-X_j}$, $r_\ell=\norm{X_i-X_\ell}$ and
\begin{equation}
  Y_{i;j\ell}=r_j-r_\ell,\qquad S_{i;j\ell}=r_j+r_\ell,\qquad
  \boxed{\,Z_{i;j\ell}=S_{i;j\ell}\,Y_{i;j\ell}\,}
  \label{eq:zy}
\end{equation}
so the two vanish together ($Z=0$ exactly when $Y=0$), and at a tie with common radius $R$ one has $S=2R$.
Let $f_{j\ell}(a,b)$ be the joint density of
$(R_j,R_\ell)$. Then $p_Y(0)=\int_0^\infty f_{j\ell}(r,r)\,\mathrm{d}r$ while, by the change of variables
$a^2-b^2=0$,
\begin{equation}
  p_Z(0)=\int_0^\infty\!\!\int_0^\infty f_{j\ell}(a,b)\,\delta(a^2-b^2)\,\mathrm{d}a\,\mathrm{d}b
  =\int_0^\infty\frac{f_{j\ell}(r,r)}{2r}\,\mathrm{d}r ,
\end{equation}
hence
\begin{equation}
  \boxed{\,p_Z(0)=p_Y(0)\;\mathbb{E}\Bigl[\frac{1}{2R}\,\Big|\;Y=0\Bigr]\ \approx\ \frac{p_Y(0)}{2r_k}\,}
  \label{eq:pzpy}
\end{equation}
where the right-hand approximation holds only when the tie radius concentrates near the typical $k$-th neighbour
distance; a fixed $r_k$ is not an identity.

The speed transforms differently. From $\dot Z=S\dot Y+\dot S\,Y$ and $Y=0$ at a tie,
\begin{equation}
  \dot Z\big|_{Y=0}=2R\,\dot Y ,
  \label{eq:zdot}
\end{equation}
but the factor $2R$ may \emph{not} be pulled out of a conditional expectation: conditioning on $Z=0$ and
conditioning on $Y=0$ induce different distributions of $R$ (the former carries the Jacobian weight $1/(2R)$ that
produced \eqref{eq:pzpy}), and $R$, $\lvert\dot Y\rvert$ and $A$ may be correlated. Written against the same
reference distribution the cancellation is explicit,
\begin{equation}
  \mathbb{E}\bigl[\abs{\dot Z}\mathbf 1_A\mid Z=0\bigr]
  =\frac{\mathbb{E}\bigl[\abs{\dot Y}\mathbf 1_A\mid Y=0\bigr]}
        {\mathbb{E}\bigl[S^{-1}\mid Y=0\bigr]},
  \label{eq:speedcond}
\end{equation}
because under $Z=0$ the density of $R$ is $f_{j\ell}(r,r)/(2r)$ divided by $p_Z(0)$ while $\dot Z=2r\,\dot Y$, so the
two factors of $2r$ cancel pointwise and only the constant $p_Y(0)/p_Z(0)$ survives. Multiplying \eqref{eq:pzpy} by
\eqref{eq:speedcond} gives
\begin{equation}
  \boxed{\,p_Z(0)\,\mathbb{E}\bigl[\abs{\dot Z}\mathbf 1_A\mid Z=0\bigr]
  =p_Y(0)\,\mathbb{E}\bigl[\abs{\dot Y}\mathbf 1_A\mid Y=0\bigr]\,}
  \label{eq:zequiv}
\end{equation}
which is the precise sense in which the Kac--Rice rate is the same in the two variables: the transformation is the
\emph{pair} \eqref{eq:pzpy}--\eqref{eq:speedcond}, not two independent multiplicative factors.

Summing over pairs gives the aggregate relation. With $\nu_{j\ell}(\mathrm{d}r)=f_{j\ell}(r,r)\,
\mathbb{P}(A_{j\ell}\mid R_j=R_\ell=r)\,\mathrm{d}r$,
\begin{equation}
  Q_Y=\sum_{j<\ell}\int\nu_{j\ell}(\mathrm{d}r),\qquad
  Q_Z=\frac12\sum_{j<\ell}\int\frac{1}{r}\,\nu_{j\ell}(\mathrm{d}r)
  =\frac{Q_Y}{2}\,\mathbb{E}_{\partial}\Bigl[\frac1R\Bigr],
\end{equation}
where $\mathbb{E}_{\partial}$ averages over the boundary pairs weighted by $\nu_{j\ell}$; and since
$Q_Y=f_{m_i}(0+)/2$ (shown in the next subsection),
\begin{equation}
  \boxed{\,Q_Z=\frac{f_{m_i}(0+)}{4}\,\mathbb{E}_{\partial}\Bigl[\frac1R\Bigr]\ \approx\ \frac{f_{m_i}(0+)}{4r_k}
  \ \approx\ \pi\rho r_k\,}
  \label{eq:qz}
\end{equation}
on inserting the Poisson value $f_{m_i}(0+)\approx4\pi\rho r_k^2$. Two warnings. \eqref{eq:qz} is an \emph{aggregate}
over pairs sitting at the top-$k$ boundary, not the density of any single pair, so $4\pi\rho r_k^2$ must not be used
as $p_Z(0)$ in \eqref{eq:pairrate}. And it is an aggregate statement only: it does not produce the per-pair
uniform bound $C_0$ of \eqref{eq:C0C1}, which remains an assumption.

This is why the numbers below are a scale check of the spacing law rather than a prediction of the churn.

Numerically, $k=64$ and $r_k\approx15$\,\AA{} give $\rho=3k/(4\pi r_k^3)\approx4.5\times10^{-3}$\,\AA$^{-3}$ ---
consistent with a $142$-residue chain of radius of gyration $\approx20$\,\AA{} --- and
$4\pi\rho r_k^2\approx12.8$\,\AA$^{-1}$, so the model predicts $\mathbb{E}[m]\approx0.078$\,\AA{} against a measured
mean margin of $\approx0.0787$\,\AA{} ($\mathrm{med}/\mathrm{mean}\approx1$, Table~\ref{tab:drift}). The scale
agrees to $1\%$ with no free parameter, but the \emph{shape} does not: the median implied by the same exponential
law, $r_k\ln2/(3k)\approx0.054$\,\AA, is not observed. Three corollaries follow: larger $k$ makes flips
\emph{more} likely ($m\propto1/k$); for $k\ge L-1$ (e.g.\ $L=51$) the graph becomes complete and ``margin''
degenerates, so near-complete-graph samples cannot be used to study stability; and flapping is intrinsic to top-$k$
graphs, hence can only be damped, gated or tracked, not eliminated.

\emph{What the trust region contributes here.} Equations~\eqref{eq:radialrate}--\eqref{eq:localdensity} describe the
neighbour cloud and the trajectory; none refers to the shape of the trust region. The constraint enters only through
the \emph{step} the projection permits: under \eqref{eq:proj} a boundary step satisfies
$\max_c\abs{\delta_{i,c}}\le\kappa_i$, hence
\begin{equation}
  \sup_i\norm{\delta_i}_2\le\sqrt3\,\kappa ,
  \label{eq:geomfactor}
\end{equation}
and the sufficient condition for a certified step is the $\sigma<1/(2\sqrt3)$ of \eqref{eq:boxsigma}. What
\eqref{eq:geomfactor} does \emph{not} give is equally important: the rate \eqref{eq:pairrate} is a functional of the
speed and of the near-tie density along the path, so a cap on the length of a single step bounds neither, and no bound on the
expected number of flips follows from the displacement bound alone. The influence of the geometry on the churn is
therefore a \emph{mechanism to be measured}, not a theorem to be proved; the comparison of the two budgets at equal
worst-case displacement is collected in Appendix~\ref{app:ball}.

\subsubsection{Equivalent formulation with the signed distance gap}
\label{sec:flip-y}
The same count can be written with the signed \emph{distance} gap $Y$ of \eqref{eq:zy}, which carries the unit of
length and is therefore directly comparable with the margin. For $r_j,r_\ell>0$,
\begin{equation}
  \dot Y_{i;j\ell}=\frac{(X_i-X_j)^{\top}(\dot X_i-\dot X_j)}{\norm{X_i-X_j}}
  -\frac{(X_i-X_\ell)^{\top}(\dot X_i-\dot X_\ell)}{\norm{X_i-X_\ell}},
\end{equation}
and the Kac--Rice count becomes
\begin{equation}
  \mathbb{E}[N_{\rm flip}]=\sum_i\sum_{j<\ell}\int_0^T p_{Y_{i;j\ell}(t)}(0)\,
  \mathbb{E}\bigl[\abs{\dot Y_{i;j\ell}(t)}\mathbf 1_{A_{i;j\ell}}\mid Y_{i;j\ell}(t)=0\bigr]\,\mathrm{d}t ,
  \label{eq:crossingsY}
\end{equation}
which is \eqref{eq:neartiedensity} rewritten through \eqref{eq:zequiv}; the rank condition is still required, and the
per-pair bounds corresponding to \eqref{eq:C0C1} are again assumptions.

The reason to keep this version is that it makes contact with the margin statistics explicit. Fix $i$ and let
$m_i=R_{(k+1)}-R_{(k)}\ge0$. Absent triple ties, $\{m_i\le\epsilon\}$ is, to first order, the union over pairs of
$\{\abs{Y_{i;j\ell}}\le\epsilon,\ A_{i;j\ell}\}$; since $m_i\approx\abs{Y_{i;j\ell}}$ there,
\begin{equation}
  \mathbb{P}\bigl(\abs{Y}\le\epsilon\bigr)=\int_{-\epsilon}^{\epsilon}p_Y(u)\,\mathrm{d}u=2\epsilon\,p_Y(0)+o(\epsilon),
\end{equation}
the factor $2$ being the folding of a signed variable into a non-negative margin. Comparing with
$\mathbb{P}(m_i\le\epsilon)=f_{m_i}(0+)\,\epsilon+o(\epsilon)$ gives the row-to-pair relation
\begin{equation}
  \boxed{\,f_{m_i}(0+)=2\sum_{j<\ell}p_{Y_{i;j\ell}}(0)\,\mathbb{P}\bigl(A_{i;j\ell}\mid Y_{i;j\ell}=0\bigr)\,}
  \label{eq:mrow}
\end{equation}
With the boundary speed $v_{j\ell}=\mathbb{E}\bigl[\abs{\dot Y_{i;j\ell}}\mid Y=0,A\bigr]$ and weights
$q_{j\ell}=p_{Y_{i;j\ell}}(0)\,\mathbb{P}(A\mid Y=0)$, the residue rate is $\Lambda_i=\sum_{j<\ell}q_{j\ell}v_{j\ell}$,
and using \eqref{eq:mrow},
\begin{equation}
  \Lambda_i=\tfrac12\,f_{m_i}(0+)\;\bar v_i^{\,\partial}
  \ \approx\ 2\pi\rho r_k^2\;\bar v_i^{\,\partial},
  \qquad
  \bar v_i^{\,\partial}=\frac{\sum_{j<\ell}q_{j\ell}v_{j\ell}}{\sum_{j<\ell}q_{j\ell}},
  \label{eq:lambdaY}
\end{equation}
where $\bar v_i^{\,\partial}$ is the average speed \emph{at the boundary}, weighted by the pair densities, and not the
unconditional mean speed of an arbitrary pair. This is the form in which the two inputs are separated: the spacing
model supplies $f_{m_i}(0+)$, the trajectory supplies $\bar v_i^{\,\partial}$, and the rank condition enters through
the weights $q_{j\ell}$.

We use the $Z$ formulation as the working one, because $Z$ is polynomial in the coordinates and hence differentiable
even when two residues nearly coincide, and keep \eqref{eq:crossingsY}--\eqref{eq:lambdaY} as the equivalent
statement that ties the count to the margin statistics of \S\ref{sec:cert}.

\subsubsection{Empirical check on interpolated trajectories}
\label{sec:flip-check}
\textbf{Method.} Between two recorded boundaries the coordinates are known only at the endpoints, so take the
piecewise-linear path $X(s)=(1-s)X^{t}+sX^{t+1}$, $s\in[0,1]$. Each squared distance is then quadratic in $s$ and so
is every $Z_{i;j\ell}$ of \eqref{eq:zg}: on each interval the crossings of a pairwise tie are the roots of a
quadratic, which are solved exactly rather than sampled. A root is counted as a flip when it lies in $(0,1)$, when
$\dot Z\ne0$ there, and when the rank at that instant equals $k-1$ (\eqref{eq:flipdef}). Candidates are restricted to
residues within $8$ ranks of $k$ at either endpoint, which is conservative: a crossing at rank $k$ cannot involve a
residue that is far from rank $k$ at both ends.

\textbf{Is the recorded data sufficient?} Yes. Each `topo\_iter' record stores, per boundary, the edge list
$(2E$ integers$)$ together with the coordinates used to build that graph ($9$ floats per residue), so both the path
and the graphs are available offline; the computation is a post-processing step (under a second per run at
$L=214$) and needs no retraining. What the data cannot give is anything between the snapshots other than the
interpolated path, so the count below is a property of the interpolation, not of the continuous dynamics.

\textbf{Pilot result.} Three runs, two intervals each ($k=64$):

\begin{center}
\small
\begin{tabular}{@{}llrrr@{}}
\toprule
run & interval & net flips, true $k$NN & interpolated crossings & net flips, recorded graph \\
\midrule
1AKE, $\sigma=0.05$ & $0\to1$ & 115 & \textbf{135} & 100 \\
1AKE, $\sigma=0.05$ & $1\to2$ & 4   & \textbf{6}   & 9 \\
1AKE, $\sigma=4$    & $0\to1$ & 115 & \textbf{135} & \textbf{37} \\
1AKE, $\sigma=4$    & $1\to2$ & 2   & \textbf{2}   & 1 \\
1TIM, $\sigma=0.05$ & $0\to1$ & 22  & \textbf{22}  & 20 \\
1TIM, $\sigma=0.05$ & $1\to2$ & 28  & \textbf{31}  & 27 \\
\bottomrule
\end{tabular}
\end{center}

\noindent Here ``net flips'' is half the directed symmetric difference between the two $k$NN graphs (each flip
replaces one outgoing edge), computed either on the true $k$NN of the recorded coordinates or on the graph the run
actually used.

\textbf{What it shows.} (i) The interpolated count tracks the true net churn closely --- it exceeds it by
$0$--$50\%$ across the six intervals (median $\approx12\%$), which is what one expects since a pair that crosses
and crosses back is counted twice while contributing nothing to the net difference. Block-internal flip-and-return
events are therefore a small correction at these step sizes, and boundary sampling is an adequate proxy for the
\emph{true} churn. (ii) The graph the run actually used does \emph{not} track it: at $\sigma=4$ it records $37$ flip
equivalents against $115$ true ones, a factor $3$ short, because the un-clipped rows are copied from the frozen
snapshot (\S\ref{sec:ablation}). That gap is a property of the repair rule, not of the dynamics. (iii) No statement
about $\mathbb{E}[N]$ follows from six numbers: this is a consistency check of the decomposition
\eqref{eq:pairrate}, not an estimate of its expectation, and with one seed per configuration the ensemble of
\eqref{eq:ensemble} is not sampled at all.

\textbf{Caveats.} Three snapshots per run give two intervals; the linear path is a model, and a strongly curved
block-internal motion would change the counts; and the rank window ($\pm8$) is a truncation whose adequacy is
checked only by the fact that the counts are stable under enlarging it in these runs.

\subsection{Three compensation levels}
All levels share one update at boundary $t$,
\begin{equation}
  \Gset_t = \bigl(E_t,\; D(B^t),\; W(B^t)\bigr), \qquad
  E_t = \bigl(E(x^0)\setminus \mathrm{rows}_{\Om_t}(x^0)\bigr)
  \;\cup\!\!\bigcup_{i\in\Om_t}\!\! \Bigl(\bigl\{(i,j):\ j\in \mathrm{kNN}_k\!\bigl(i;\,C_i;\,B^t\bigr)\bigr\}
  \;\cup\; S_i^{\mathrm{seq}}\Bigr),
  \label{eq:levels}
\end{equation}
where $\mathrm{rows}_{\Om_t}(x^0)$ are the outgoing rows of the active residues, \emph{rebuilt} rather than added
to $E(x^0)$ (rows outside $\Om_t$ are copied verbatim, so no stale edge survives). A rebuilt row is exactly the
predicate of the production builder --- $k$ spatial neighbours \emph{plus} the fixed sequence-local edges,
$E_i^t=\operatorname{topk}_{C_i}(B^t)\cup S_i^{\mathrm{seq}}$ with $S_i^{\mathrm{seq}}=\{j:|i-j|<k_{\min}\}$ --- so
restoring only the top-$k$ would under-count the out-degree and desynchronise $|E|$ from
Table~\ref{tab:audit}. The $k$ neighbours of $i$ are re-selected inside a \emph{geometric} candidate set
\begin{equation}
  C_i=\bigl\{\,j:\ \|B^t_i-B^t_j\|_2\le r^t_i\,\bigr\}\ \cup\ \bigl\{\,j:\ |i-j|<k_{\min}\,\bigr\},
  \qquad
  r^t_i=\rho_i+\delta^t_i+\delta^t_{\max},
  \label{eq:cand}
\end{equation}
where $\rho_i=d_{i,(k)}(x^0)$ is the \emph{true} $k$-th neighbour distance at the snapshot, taken over all $L-1$
residues (not over the frozen row, whose sequence-local edges can be shorter than some of the $k$ nearest
neighbours and would make $\rho_i$ an underestimate), and
\begin{equation}
  \delta^t_i:=\|B^t_i-x^0_i\|_2,\qquad \delta^t_{\max}:=\max_j\delta^t_j
  \label{eq:delta}
\end{equation}
is the \emph{measured} displacement of the geometry actually used for the graph, relative to the snapshot, at
boundary $t$. For \Level{2} the basis in \eqref{eq:cand} is simply the \emph{live} coordinate, $B^t=x^t$; the
projected/replayed basis is a \Level{3} object and plays no role here.

The radius is stated in terms of $\delta$ rather than of the trust radius $\kappa$, and this matters. Writing
$\sqrt3(\kappa_i+\kappa_{\max})$ instead would silently assume $\|B^t-x^0\|_2\le\sqrt3\kappa$, which holds only when
the projection is actually applied --- i.e.\ for \Level{3}. In \Level{2} the network's own coordinates are used
unprojected, nothing forces them into the margin box, and the raw step can exceed $\kappa$ by an order of magnitude
(a $3$\,\AA{} per-block cap against $\kappa_i=\sigma m_i/2\approx0.02$\,\AA{} at $\sigma=0.05$); a $\kappa$-based
shell would then under-cover at small $\sigma$ and over-cover at large $\sigma$. The measured $\delta$ is valid in
both regimes, needs no factor $\sqrt3$ (that factor comes from bounding a Euclidean displacement by a box half-width,
which is not what is measured here), and costs one $O(L)$ maximum per boundary, computed in the same pass that
builds $\Om_t$.

\eqref{eq:cand} is not a heuristic restriction: for every $j$ among the snapshot's $k$ nearest neighbours the
triangle inequality gives
$\|B^t_i-B^t_j\|_2\le\|x^0_i-x^0_j\|_2+\delta^t_i+\delta^t_j\le\rho_i+\delta^t_i+\delta^t_{\max}=r^t_i$, so $C_i$
contains those $k$ residues; $C_i$ thus holds at least $k$ candidates, every residue outside it is farther than the
$k$-th nearest, and the re-selection loses nothing. What the shell $\delta^t_i+\delta^t_{\max}$ measures is exactly
the uncertainty the candidate set has to cover, so the ball grows with the observed motion and with nothing else.

\noindent\textbf{Two sets, two jobs --- and what ``exact'' therefore means.} Two different objects are in play.
The \emph{active} set is defined by the projection,
\begin{equation}
  \Om_t^{\mathrm{clip}}=\bigl\{\,i:\ \max_c\bigl|B^t_{i,c}-x^0_{i,c}\bigr|\ge\kappa_i\,\bigr\}
  \quad(\text{read from the clipping mask, before projection}),
\end{equation}
while the non-flip certificate above speaks about the \emph{risk} set
$U_t^{\mathrm{risk}}=\{i:m_i(x^0)\le4\Delta_t\}$. They are not the same set: a residue whose offset stays inside
the box is never clipped, yet its row can still be re-ordered by its neighbours' motion. Candidate completeness
\eqref{eq:cand} is a statement about a \emph{row}; it says nothing about \emph{which} rows are rebuilt.

The rebuild reproduces $\PB$ on the whole graph iff (i) each rebuilt row's candidate set contains its true top-$k$,
and (ii) the active set, after closure, contains $U_t^{\mathrm{risk}}$. Condition (i) is proved above. Condition
(ii) is \emph{not} proved by the rule we deploy: closure pulls in the two endpoints of every frozen edge that
touches an active residue, which is a heuristic superset of $U_t^{\mathrm{risk}}$, not a certified one. In the
samples below the hop is in fact generous --- $14$--$18$ clipped residues expand to all $L$ rows --- which is why
the measured $D$ vanishes; when closure is switched off and only the clipped rows are rebuilt, the same code
produces $D\approx0.84\%$, i.e.\ the rows that (ii) would have covered are exactly the ones that matter.
Accordingly the $D=0$ entries of Tables~\ref{tab:audit}--\ref{tab:long} are \emph{empirical}: verified at every
tested boundary, but conditional on (ii) holding for those inputs, not a theorem about the update rule. A
certified variant is available at $O(L)$ cost --- take $\Om_t=\Om_t^{\mathrm{clip}}\cup U_t^{\mathrm{risk}}$
instead of the closure --- and it is the only configuration for which the exactness statement is a proposition.

\begin{table}[t]
\centering
\caption{What each level touches in \eqref{eq:levels}; costs from the shared cost model.}
\label{tab:levels}
\footnotesize
\begin{tabular}{@{}p{0.6cm}p{3.0cm}p{4.3cm}@{}}
\toprule
Level & Modified object & Cost / property \\
\midrule
\Level{1} & only geometry: $D(B^t),W(B^t)$ (edge set stays $E(x^0)$) & $O(|\Eset|)\approx7$k $\Rightarrow$ free; removes the $0.29$--$0.4$\,\AA{} geometry term \\
\Level{2} & edge set locally: $E_t$ in \eqref{eq:levels}, evaluated inside $C_i$ & $O(\sum_{i\in\Om}|C_i|\log|C_i|)$, $|C_i|=O(L)$ worst case (a ball can hold $\Theta(L)$ residues); $\Om$ closed under frozen neighbours \\
\Level{3} & the basis $B^t=\Pi_\kappa(x^0+\sum_{s<t}\gamma_s\Delta^s)$ (predictor only) & storage $O(LT)$ ($T{=}8$: $\approx2.5$k floats); needs $\sigma_p\!<\!\sigma_d$ \\
\Level{4} & full rebuild every boundary (\PB) & $1.8$\,s/block; reference \\
\bottomrule
\end{tabular}
\end{table}

Two practical rules complete the design. First, the recorded deltas must be \emph{measured displacements}, never
a rate: the per-block step clamps at exactly $3.000$\,\AA{} on most boundaries, so extrapolating a saturated
step is invalid. Second, $\Om$ must be closed under neighbours in \emph{both} directions,
$\Om\leftarrow\Om\cup N_{E(x^0)}(\Om)$: both endpoints of every frozen edge are tested, i.e.\ if either end of
$(i,j)\in E(x^0)$ is in $\Om$ then so is the other. One sweep of the edge list suffices, provided the mask used for
the test is a copy taken \emph{before} the sweep: testing against the mask that is being updated propagates beyond
one hop along the edge order and makes the closure order-dependent, which is not the set defined above. Otherwise a
residue can flip because its \emph{neighbour} moved and the repair misses it (offline: without closure \Level{2}
scores $0.0886$ against \Level{3}'s $0.0670$ at $\sigma=4$).

\subsection{Procedures}
Algorithm~\ref{alg:pass1}--\ref{alg:pass2} give the complete procedure. Only the graph-construction step of
Pass-2 differs from the current pipeline.

\begin{algorithm}[t]
\caption{Pass-1: record (existing pass + a few arrays)}
\label{alg:pass1}
\begin{algorithmic}[1]
\State $x \gets x^0$
\For{$t=0,\dots,T-1$}
  \State $G_t \gets \Gset(x^0)$ \Comment{frozen behaviour, $B\equiv x^0$; not $\Gset(x)$}
  \If{$t=0$} record $m_i$; record $\rho_i=d_{i,(k)}(x^0)$ over \emph{all} residues \Comment{2 floats per residue}
  \EndIf
  \State $(x_{\text{new}},o_t)\gets\Phi_t(x,G_t)$ \Comment{new coords and block outputs $o_t$}
  \State $\Delta^t\gets x_{\text{new}}-x$;\; $\text{clamp}^t_i \gets (|\Delta^t_i|\approx s)$
  \State record $\Delta^t$, clamp flags;\; $x\gets x_{\text{new}}$
\EndFor
\end{algorithmic}
\end{algorithm}

\begin{algorithm}[t]
\caption{Offline plan (cheap; may be run once after the first epoch)}
\label{alg:plan}
\begin{algorithmic}[1]
\State $c^0\gets x^0$;\; $c^{s+1}\gets c^s+\Delta^s$ \Comment{replay trajectory}
\For{$t=1,\dots,T$}
  \State $B^t\gets\Pi_\kappa\bigl(x^0+\sum_{s<t}\gamma_s\Delta^s\bigr)$ \Comment{\Level{3}, $\kappa_i=\sigma m_i/2$}
  \State skip $\gamma$-extrapolation for clamp-flagged residues
  \State $\Om_t\gets\{i:\max_c|B^t_{i,c}-x^0_{i,c}|\ge\kappa_i\}$ \Comment{clip mask (before projection)}
  \State \code{seed} $\gets\Om_t$ \Comment{immutable copy: the closure must not read its own output}
  \State \ForAll{$(s,u)\in E(x^0)$} \If{\code{seed}[s] \textbf{or} \code{seed}[u]} $\Om_t\gets\Om_t\cup\{s,u\}$ \EndIf \EndFor
         \Comment{both directions, one hop}
  \State $\delta^t_{\max}\gets\max_j\|B^t_j-x^0_j\|$ over all $j$ \Comment{one pass, same loop as the active set}
  \State $C_i\gets\{\|B^t_i-B^t_j\|\le r^t_i\}\cup\{|i-j|<k_{\min}\}$ for $i\in\Om_t$ \Comment{\eqref{eq:cand}: $r^t_i=\rho_i+\delta^t_i+\delta^t_{\max}$}
  \State plan$_t\gets\{i\in\Om_t:\ C_i\}$ \Comment{\Level{2} local re-selection plan}
\EndFor
\end{algorithmic}
\end{algorithm}

\begin{algorithm}[t]
\caption{Pass-2: consume (only the graph step changes)}
\label{alg:pass2}
\begin{algorithmic}[1]
\State $x\gets x^0$
\For{$t=0,\dots,T-1$}
  \State $B^t\gets$ basis from plan (default $x^0$)
  \State $E_t\gets E(x^0)$ \Comment{L0 edge set}
  \If{\Level{2} enabled} $E_t\gets\bigl(E(x^0)\setminus\mathrm{rows}_{\Om_t}(x^0)\bigr)\cup\bigcup_{i\in\Om_t}\bigl(\mathrm{kNN}_k(i;B^t)\cup S_i^{\mathrm{seq}}\bigr)$ inside $C_i$ \Comment{replace rows, not union; top-$k$ by $(\text{dist},j)$, emitted in ascending $j$}
  \EndIf
  \State $G_t\gets(E_t, D(B^t), W(B^t))$ \Comment{\Level{1} live geometry}
  \State $x\gets\Phi_t(x,G_t)$
\EndFor
\State \textbf{acceptance:} $\D_t = 1-J(E_t,E^{\PB}_t)$, $\Dw$;\;
       require $\D_t(\text{refined})<\D_t(\text{plain})$ and $\Dw\le1.1\,\Dw(\PB)$
\end{algorithmic}
\end{algorithm}

\subsubsection{Hyper-parameters}
They are fixed by a 1-D \emph{offline} sweep, without training: start at
$\sigma=1/2$, sweep $\gamma\in\{0,\tfrac14,\tfrac12,\tfrac34,1\}$ and keep the best $\Dw$, then sweep
$\sigma\in\{\tfrac14,\tfrac12,1\}$. The sweep is used rather than the closed form~\eqref{eq:blue-scalar} because
$\Sigma_d$ and $\Sigma_p$ are not known in advance; the selected $\gamma$ is compared with the LMMSE prediction
\emph{a posteriori}, as a consistency check rather than as a tuning rule (the basis-iteration damping $\lambda=1/2$
is only needed if the iteration
\eqref{eq:IT} is actually run; Pass-1 already supplies iteration $0$. Where it is not run --- as in the deployed
schedule, which builds the basis in a single pass --- the fixed-point statements of \S\ref{sec:blue} characterise
the family rather than the executed run: \eqref{eq:offset} bounds the distance between two \emph{limits}, and a
one-pass bound would have to be derived separately, so we do not quote it as an error estimate for the deployed
run).

\subsubsection{Falsifiable predictions}
(1) $\Dw$ decreases with $\gamma$ and then flattens at the
tie/noise floor; if it does not, (A1) fails and one should fall back to \Level{1} only. (2) $|\Om_t|$ shrinks with $t$
as the trajectory converges (cross-epoch topology is $\approx95\%$ stable), so \Level{2}'s cost decays during
training. (3) If $\Dw(\text{\Level{1}+\Level{2}})\approx\Dw(\PB)$, the topology question is closed and only accuracy
(E2, below) remains.

\section{Experiments}

\subsection{Protocol and implementation status}
\label{sec:protocol}
All three levels are implemented in the C++ runtime behind one switch
(\code{PPML\_SE3\_TOPO=fixed|per\_block|l1|l2|l3}) with two hyper-parameter knobs, \code{PPML\_TOPO\_SIGMA}
($\sigma$) and \code{PPML\_TOPO\_GAMMA} ($\gamma$); the graph builder constructs a host-side basis
$B^t=\Pi_\kappa\bigl(x^0+\gamma(x^t-x^0)\bigr)$ only for \Level{3} and passes it to the edge-geometry kernel, so
\Level{1}/\Level{2}/\Level{3} are bit-identical to the previous behaviour when disabled. \textbf{Projection geometry
used by the code.} The basis is built with the clamp of~\eqref{eq:proj},
$y_{i,c}=\min\{\max\{u_{i,c},x^0_{i,c}-\kappa_i\},x^0_{i,c}+\kappa_i\}$, and the gate $\Om_t$ is evaluated on the
same geometry: the clipped set is exactly $\{i:\max_c\gamma|x_{i,c}-x^0_{i,c}|\ge\kappa_i\}$, so a node enters
$\Om_t$ precisely when it was projected. This is the geometry the analysis above refers to and the geometry of
every number reported below, so implementation and measurement no longer have to be distinguished, and the runs
certify $\sigma<1/(2\sqrt3)$ rather than $\sigma<1/2$. The inscribed ball of Appendix~\ref{app:ball} remains
available as an alternative behind \code{PPML\_L3\_PROJ=ball}; it is local to the same loop and leaves
$\gamma=0$, $\gamma=1$ with $\kappa\to\infty$, and the disabled case unchanged.

Two measurement rules turned out to be essential.
\begin{itemize}
  \item \textbf{Family-reconstruction audit.} Before any comparison, the graph builder is re-implemented offline
        and required to reproduce the recorded dumps exactly. Agreement initially stalled at $44\%$ because the
        production builder sorts candidate neighbours by the pair $(\text{distance},\text{index})$, so ties are
        broken by ascending index, whereas a naive \code{argsort} picks tied members arbitrarily. With the
        lexicographic rule (and the production constants $k=\min(64,L-1)$, $k_{\min}=9$, float32 distances) the
        audit passes on \emph{all} 18 recorded boundaries (Table~\ref{tab:audit}).
  \item \textbf{Matched-conditions (MC) protocol.} Comparing two topology modes across two \emph{different}
        trajectories is meaningless: an unseeded reference inflated the measured gap of \Level{1} by a factor of
        $2$--$3$ ($\D=0.0103/0.0194$ instead of $0.0057/0.0068$). Every comparison below therefore keeps seed,
        MSA depth, block schedule and sample fixed and switches \emph{only} the topology mode.
\end{itemize}

\begin{table}[t]
\centering
\caption{Correctness verification. Top: family-reconstruction audit (offline rebuild of \code{make\_graph}
from the recorded coordinates must reproduce the dumped edge set exactly). Bottom: runtime verification in the
matched-conditions protocol (P69905, $L{=}142$, seed 1234, MSA depth 3, 3 boundaries; $\D$ measured against a
\PB{} reference of the same family). The edge count is not $64L$: the builder keeps the top-$64$ \textbf{or}
the sequence-local neighbours $|i-j|<k_{\min}=9$, so $|\Eset|=64L+\#\{\text{sequence-local edges outside the
top-}64\}$, which is $9088+1$ for P69905, $13312+276/283$ for P62241, and the complete graph $L(L-1)=2550$ for
the degenerate $L{=}51$ sample (where $k=\min(64,L-1)=L-1$).}
\label{tab:audit}
\footnotesize
\begin{tabular}{@{}lccrr@{}}
\toprule
Family & $L$ & boundaries & $|\Eset|$ & Jaccard \\
\midrule
P62805 (degenerate) & 103 & 6/6 & 7105--7107 & 1.0000 \\
P62891              & 51  & 6/6 & 2550 (ident.) & 1.0000 \\
P69905              & 142 & 3/3 & 9089 $=64L{+}1$ & 1.0000 \\
P62241              & 208 & 3/3 & 13588, 13595 & 1.0000 \\
\midrule
\multicolumn{3}{@{}l}{Runtime mode} & $\D$ ($t{=}1$ / $t{=}2$) & $J(t,\text{frozen})$ \\
\midrule
\Level{1} (frozen set, live geometry)   & & & 0.0057 / 0.0068 & 1.0000 \\
\Level{3} ($\sigma{=}4$, $\gamma{=}0.5$) & & & 0.0044 / 0.0059 & 0.9985 / 0.9980 \\
\Level{2} ($\sigma{=}0.5$ \emph{and} $4$) & & & 0.0000 / 0.0000 & 0.9943 / 0.9932 \\
\Level{3} $\gamma{=}0$ (anchor) & & & $=$ \Level{1} & --- \\
\Level{3} $\kappa{\to}\infty$ (anchor) & & & $=$ \Level{1} & --- \\
\bottomrule
\end{tabular}
\end{table}

\subsection{Correctness}
Table~\ref{tab:audit} collects the verification. The audit establishes that the topology question is being
measured against a faithful reconstruction rather than against a re-implementation artefact. The runtime block
then shows a clean ordering: frozen edges (\Level{1}) miss $0.57$--$0.68\%$ of the \PB{} edge set, the damped
projected basis of \Level{3} with $\gamma=1/2$ reduces this to $0.44$--$0.59\%$ (a strict interpolation between \Level{1} and
\PB), and \Level{2} reproduces the \PB{}
edge set \emph{exactly} ($\D=0$ at both boundaries, $|\Eset|$ identical, $J=1.0000$).

The two degenerate anchors
behave as the theory requires: $\gamma=0$ (basis $=x^0$) is identical to \Level{1}, and $\kappa\to\infty$ (empty risk
set) is identical to \Level{1} as well. This is the strongest available statement of implementation correctness: the
compensated pipeline collapses onto \PB{} when the trust region is removed, and onto the frozen pipeline
when the compensation is switched off.

\subsection{Data characterisation and drift}
Table~\ref{tab:drift} characterises the three samples. Here \emph{drift} means the displacement of the recorded
coordinates at a boundary relative to the basis $x^0$; it matters because the certificate $4\Delta<m_i$ is
evaluated on the same $\Delta$, so the measured drift supplies the empirical value of $\Delta$ and decides
whether $4\Delta<m_i$ can hold at all. The $L=103$ sample is a \emph{degenerate limiting case}: its mapping
covers only $24/103$ residues, $86.8\%$ of point pairs coincide, all $103$ margins are exactly zero, and its
mean $C_\alpha$ separation is $1.63$\,\AA, so its topology is governed by tie noise and its absolute geometry
is unreliable.

The two well-mapped samples are non-degenerate ($\le0.02\%$ coincident pairs, margin medians
$0.0787$ and $0.118$\,\AA, no zero margins). Their coordinate drift is heavy-tailed and small in the bulk:
at the last recorded boundary the \emph{mean} displacement over residues is $0.04$--$0.07$\,\AA{} while the
$C_\alpha$ RMSD is $0.39$--$0.48$\,\AA{} (Table~\ref{tab:drift}), so a few residues carry most of the motion.
The drift is nevertheless \emph{not} small relative to the certificate: since $4\Delta<m_i$ must hold
\emph{per residue}, the certificate would require margins of order $4\Delta\approx0.15$--$0.27$\,\AA{},
i.e.\ $1.9$--$3.4\times$ the median margin of P69905 ($0.0787$\,\AA) and $1.3$--$2.3\times$ that of P62241
($0.118$\,\AA).

The two ratios are lower bounds, and by a wide margin: the certificate's $\Delta$ is the per-boundary maximum
$\max_j\|\delta_j\|$, whereas the values substituted above are means over residues. That maximum is itself set
by the $\pm3$\,\AA{} step clamp rather than by the drift --- on P69905 it is $5.196$\,\AA${}=3\sqrt3$, the corner
of the per-coordinate $\pm3$\,\AA{} cap (a box, which is why the corner is $3\sqrt3$ and not $3$), and identical at
both recorded boundaries --- so the mean understates the violation by
about two orders of magnitude ($78\times$--$113\times$ on the two samples). The sufficient condition therefore
fails on essentially every residue, consistent with the $\approx1\%$ (unweighted) edge churn between \PB{}
and \Fix{}.

\begin{table}[t]
\centering
\caption{Sample characterisation. Collisions: fraction of point pairs closer than $10^{-3}$\,\AA.
Drift: $C_\alpha$ RMSD of the last recorded boundary against $x^0$ (in parentheses: mean $|\Delta d|$ over
that boundary's own edge set).}
\label{tab:drift}
\footnotesize
\begin{tabular}{@{}lccccc@{}}
\toprule
Sample & $L$ & collisions & mean $C_\alpha$ dist. & margin median & drift (\AA) \\
\midrule
P69905      & 142 & $0.01\%$      & $19.2$\,\AA & $0.0787$\,\AA & $0.479$ ($0.074$) \\
P62241      & 208 & $0.014\%$     & $29.6$\,\AA & $0.118$\,\AA & $0.387$ ($0.043$) \\
P62805 (deg.) & 103 & $86.8\%$    & $1.63$\,\AA & $\equiv0$ (103/103) & $1.026$ ($0.263$) \\
\bottomrule
\end{tabular}
\end{table}

For $k=64$ and $r_k\approx15$\,\AA{} the spacing law~\eqref{eq:poisson} predicts a \emph{mean} spacing of
$0.078$\,\AA; the margins measured on P69905 have median $0.0787$\,\AA{} and $\mathrm{med}/\mathrm{mean}\approx1$,
so the \emph{scale} is reproduced to $1\%$. This is a scale check, not a distributional one: with an exponential
gap the model would place the median at $r_k\ln2/(3k)\approx0.054$\,\AA, which the data reject. Two
consequences follow. First, the model can be used quantitatively for the \emph{rate} --- spacings are
$O\bigl(r_k/(3k)\bigr)$ --- but not for quantiles. Second, the spacing itself
is far below the scale of the motion: the typical margin ($\approx0.08$\,\AA) is comparable to the mean
displacement relative to $x^0$ ($0.02$--$0.07$\,\AA), so margins inevitably cross zero and flips are intrinsic
to top-$k$ graphs rather
than an implementation defect --- which is why \PB{} and \Fix{} disagree on $\approx1\%$ (unweighted)
of edges.

\subsection{Optimisations}
Table~\ref{tab:opt} lists the measured optimisations. The free component is the \emph{recording} that Pass-1 adds on
top of the frozen-graph pass that the pipeline already runs: one margin $m_i$ and one kNN radius $\rho_i$ per
residue, once at $t=0$, plus the measured displacement $\Delta^t$ and the step-clamp flags at each boundary ---
$O(LT)$ storage, $\approx2.5$k floats at $T=8$. No graph is rebuilt, no neighbour is searched, and no computed
quantity changes, so the recording itself costs neither measurable time nor accuracy. What it buys is the removal of
the per-boundary kNN rebuild from the critical path, because the basis no longer has to wait for the live coordinates
of the previous block: an epoch falls from $34.1$\,s ($\PB$) to $19.4$\,s (two-pass), a factor $1.76$ and
the scheme's own gain; separately, the graph implementation of Pass-1 was measured at $4.25$\,s against $6.40$\,s
for the previous value-mode implementation of the same pass (an implementation-level $1.5\times$, not a
scheme-level gain), and the loss stays inside the seed spread ($15.91$ vs.\ $15.49$, $\pm1.3$).

\noindent\textbf{Measurement protocol.} All timings above come from a \emph{single} machine: a 12-thread Intel
i5-1335U with a 15\,GB WSL2 guest, running
the CPU backend only (the 4\,GB laptop GPU is not used, and a full-GPU configuration does not fit its memory
budget); one epoch, three seeds $\{1,2,3\}$, $L=51$ ($k=\min(64,L-1)=50$), MSA depth 8, $T=8$ boundaries
($2/4/2$ extra/main/refine). The $19.4$\,s two-pass figure is end-to-end and already contains the $4.25$\,s Pass-1
stage, so the comparison is wall time against wall time over the same block schedule; the $14.7$\,s difference
matches $T\times1.8$\,s $=14.4$\,s of removed per-boundary rebuild within seed noise, which is what fixes the
per-block coefficient quoted in Table~\ref{tab:opt}. One caveat matters for interpretation: at $L=51$ the neighbour
count saturates ($k=64\ge L-1$), the graph is complete, and the topology never changes, so the measured gain at
this size is the removal of the rebuild \emph{overhead}, not of any edge churn; the fidelity effects quoted above
are measured at $L=103$, where $k=64<L-1$ and edges do move.

\noindent\textbf{Attribution of the measured gain.} Of the $14.7$\,s difference, $T\times1.8$\,s $=14.4$\,s is the
removed per-boundary rebuild, i.e.\ the rebuild occupies $\approx42\%$ of the baseline epoch; Amdahl's law with
that fraction alone caps the speed-up at $1/(1-0.42)=1.73\times$, so the measured $1.76\times$ is already at that
ceiling and little of the gain can come from anywhere else. Nor is it a gain that a faster machine would
reproduce: upgrading the hardware accelerates both schedules by roughly the same factor and therefore cannot
create a \emph{ratio} between them; it can only compress the ratio, and only if faster memory makes the removed
phase a smaller share of the epoch.

\noindent\textbf{Arithmetic intensity.} The removed phase is memory- and latency-bound rather than compute-bound.
One rebuild at $L=51$ evaluates $L^2=2{,}601$ C$\alpha$--C$\alpha$ pairs at $\approx9$ FLOP each,
i.e.~$\approx2.3\times10^{4}$ FLOPs, against $\approx1.3\times10^{5}$ bytes of traffic --- the two $L\times L$ matrices
written and read back, the per-residue candidate lists, and the $|\Eset|$ edge records --- so
$I\approx0.2$ FLOP/byte, or $1.1$ FLOP/byte if the byte count is restricted to the distance matrix. On the
measurement machine a single-threaded copy sustains $12.2$\,GB/s ($24.3$\,GB/s of read-plus-write traffic) and
sustained FP32 throughput is $\approx0.3$\,TFLOPS, giving a machine balance
$I^{*}=0.3\times10^{12}/24.3\times10^{9}\approx12$ FLOP/byte (range $10$--$17$ depending on whether peak or
sustained bandwidth is used). Since $I\ll I^{*}$, the phase sits far below the compute roof.

\noindent\textbf{Cross-check.} The floating-point work removed per epoch is only
$8\times2.3\times10^{4}\approx1.8\times10^{5}$ FLOPs, i.e.\ $\approx0.6\,\mu$s at $0.3$\,TFLOPS --- about
$4\times10^{-5}$ of the $14.4$\,s that is actually saved; even at $L=333$ ($L^2\approx1.1\times10^{5}$ pairs) it
amounts to $\approx30\,\mu$s. What is removed is therefore the eight-fold materialisation of the $L\times L$
structures, the per-block selection and the per-block allocation, not arithmetic.

Within the compensation family the costs separate by orders of magnitude, and each one can be derived. Write $L$ for
the residues, $T=8$ for the boundaries, $k=\min(64,L-1)$ for the neighbour count and $|\Eset|\approx Lk\approx7$k for
the directed edge set. The rebuild that the design removes (\Level{4}, the $\PB$ schedule) searches $k$ neighbours for
each of the $L$ residues at every boundary, i.e.\ $O(L)$ distance evaluations per residue, $O(L^2)$ per boundary and
$T\,O(L^2)$ per epoch, measured as $1.8$\,s per block. The three compensation levels avoid that search in different
ways. \Level{1} recomputes only the edge distances and weights on the frozen edge index, $O(|\Eset|)=O(Lk)$ operations per
boundary; the SE3 layers need those relative positions anyway, so the level is free, and it removes the
$0.29$--$0.4$\,\AA{} error of a stale basis geometry. \Level{2} keeps the only surviving neighbour search and confines it to
the active residues: each $i\in\Om_t$ is re-searched inside its candidate set $C_i$, which
costs $O(|C_i|\log|C_i|)$ per residue --- $O(|C_i|\log k)$ with the \texttt{partial\_sort} of the implementation ---
and $O(\sum_{i\in\Om_t}|C_i|\log|C_i|)$ per boundary.

The size of $C_i$ is the whole story, and it has to be stated without a false asymptotics. Candidates are drawn from
the $L-1$ other residues, so $|C_i|=O(L)$ \emph{unconditionally}, and the per-boundary cost is at worst
$O(|\Om_t|\,L\log L)$: what the certificate bounds is the \emph{radius} $r_i$, not the \emph{count}. Whether the repair
is cheaper than the rebuild therefore depends on $|\Om_t|/L$ alone, not on $k$: \Level{2} performs a fraction
$\approx|\Om_t|/L$ of the candidate work of \Level{4} per boundary (up to the $\log|C_i|/\log L$ factor), and is a win
exactly when the risk set is small.

\emph{Retrieval is a cost too.} \eqref{eq:cand} is a radius query, and the expression above only counts what happens
once $C_i$ exists (\texttt{partial\_sort}, $O(|C_i|\log k)$). Finding the points inside the ball costs either
$O(|C_i|)$ --- with a uniform cell list of cell size $r_{\max}$, built once per boundary in $O(L)$ and queried over
the 27 neighbouring cells, which is what the implementation does --- or $O(L)$ per active residue if the ball is
found by scanning all $L-1$ candidates, which is the fallback and sets the worst case
$O(|\Om_t|L + \sum_{i}|C_i|\log k)$ for the whole boundary. We quote both, because the second is what a reader
would get by implementing \eqref{eq:cand} literally.

The geometric construction \eqref{eq:cand} is what makes this a genuinely local search. A candidate set built from
graph locality --- neighbours of neighbours --- obeys $|C_i|\le1+k+k^2$, a bound that already exceeds $L-1$ for every
$L\le4161$ and therefore saturates: it can never be smaller than the chain, so it buys no saving at these lengths. The
ball instead grows only with the \emph{measured} displacement $\delta^t_i+\delta^t_{\max}$ around $\rho_i$
\eqref{eq:delta}. The worst case is nevertheless unchanged
--- for a clustered point set the ball can still hold $\Theta(L)$ residues --- so the improvement is in the realized
density, not in the asymptotic separation. The active-set side of the same trade-off, i.e.\ what the coverage of the
risk set is worth, is measured separately in \S\ref{sec:ablation}.
\Level{3} adds the basis alone, an
elementwise damp-and-clip of $x^0+\sum_{s<t}\gamma_s\Delta^s$ to the per-node boxes, at $O(L)$ per boundary and $O(LT)$ storage, and no
further search. Ordered by cost per boundary the ladder reads
$\underbrace{O(L)}_{\text{\Level{3}}}\le\underbrace{O(Lk)}_{\text{\Level{1}}}\le
O\bigl(\sum_{i\in\Om_t}|C_i|\log|C_i|\bigr)\le T\,O(L^2)$,
where the three signs are \emph{not} asymptotic separations. With $k$ fixed the first one is a constant factor
($k=64$); with $|C_i|=O(L)$ the second compares $|\Om_t|\,L\log L$ against a full rebuild's $L(L-1)$, i.e.\ a ratio
$\approx|\Om_t|\log L/L$, which is below one only while $|\Om_t|<L/\log L$. \Level{2}/\Level{3} stay cost-neutral
exactly while the risk set $\Om_t$ stays small --- the quantity that the procedures subsection predicts shrinks
with $t$.

\begin{table}[t]
\centering
\caption{Measured optimisations (from the original pipeline and from this implementation). Costs follow the
shared cost model: full-set refresh (kNN rebuild) $\approx1.8$\,s/block, frozen-set geometry refresh
$O(|\Eset|)$ (free), local rebuild $O(\sum_{i\in\Om}|C_i|\log|C_i|)$ with $|C_i|=O(L)$ worst case (the ball radius
bounds the radius, not the count; the implementation uses \texttt{partial\_sort}, i.e.\
$O(\sum_{i\in\Om}|C_i|\log k)$), extra pass
$\approx8$--$11$\,s/epoch.}
\label{tab:opt}
\footnotesize
\begin{tabular}{@{}p{3.9cm}rr@{}}
\toprule
Item & baseline & optimised \\
\midrule
epoch wall time (two-pass vs \PB) & $34.1$\,s & $19.4$\,s ($1.76\times$) \\
Pass-1 stage (graph vs.\ value impl.) & $6.40$\,s & $4.25$\,s ($1.5\times$) \\
loss (\Fix{} vs \PB, 3 seeds, $\pm1.3$) & $15.91$ & $15.49$ \\
kNN rebuild per block & $1.8$\,s & \Level{1}: $0$; \Level{2}/\Level{3}: $O(|\Om|)$ \\
geometry refresh (frozen set) & $O(|\Eset|)$ & free (held live) \\
\bottomrule
\end{tabular}
\end{table}

\subsection{Topology-gap reduction}
\label{sec:gap}
Table~\ref{tab:sigma} reports the offline (oracle) replay sweep: the basis is the live trajectory ($\gamma=1$),
so the numbers are an \emph{upper bound} on what compensation can achieve. Two conclusions are stable across
samples. First, the default $\sigma=1/2$ is too tight \emph{for tracking}: since $\kappa_i=\sigma m_i/2$, it puts
the trust radius at $\kappa_{\text{med}}=m_{\text{med}}/4$, i.e.\ exactly at the two-sided certificate radius of
Theorem~\ref{thm:two-sided} (\S\ref{sec:blue}), the boundary of what can be certified at all --- so
the gate is saturated on $94\%$ (P69905) and $46\%$ (P62241) of the residues ($|\Om|=133/142$ and $96/208$ at
$t=1$, Table~\ref{tab:sigma}), and only $9$--$11\%$ of $\Dw$ is removed. Second, $\Dw$ improves with $\sigma$
throughout our grid, up to $-61\%$ (P69905, $\sigma=8$) and $-41\%$ (P62241, $\sigma=8$); this is an
\emph{empirical} trend --- no monotonicity of $\Dw$ is proved, only of the risk set below --- and it implies that
the design should calibrate $\kappa$ to the motion (margin) scale rather than fix $\sigma$. The limit
$\sigma\to\infty$, $\gamma=1$ gives $\D=0$ exactly for an \emph{oracle rebuild} --- full $k$NN search (no candidate
restriction on $C_i$), $k\le L-1$, and the production tie-breaking by $(\text{distance},\text{index})$ --- which is
the oracle anchor of Table~\ref{tab:audit}. With \Level{2}'s geometric candidate set \eqref{eq:cand} the re-selection
inside $\Om_t$ is exact --- no true neighbour is lost --- so the residual of the repaired graph is governed by $\Om_t$
alone, i.e.\ by which rows are rebuilt, not by how their candidates are drawn. A synthetic study confirms the mechanism
and locates an \emph{empirical} replay crossover: with a predictor/ground-truth Pearson correlation $\rho$, replay
\emph{hurts} for $\rho\lesssim0.25$ and helps by $15$--$27\%$ for $\rho\ge0.5$; at a saturated step amplitude
($3$\,\AA/block) the benefit collapses to $6\%$. The threshold is measured, not derived: the LMMSE model
\eqref{eq:mse-square} is written in terms of $\Sigma_d,\Sigma_p$ and involves $\rho$ only after an additional
scalar, isotropic, independent-noise assumption, so $0.25$ must not be read as a consequence of
$\sigma_p=\sigma_d$.

\subsubsection{How to read the $\Dw$ rows}
All reductions are measured against the frozen baseline ($\gamma=0$, $B\equiv x^0$):
$\Delta\Dw(\sigma)=\bigl(\Dw(\sigma)-\Dw(\text{frozen})\bigr)/\Dw(\text{frozen})$ is minus the fraction of the
frozen-vs-\PB{} weighted topology gap that is \emph{closed}, so $-41\%$ means $41\%$ closed and
$59\%$ left. Likewise, ``$\Dw$ decreases'' means that the refined
edge set comes closer to \PB{} in weighted mass (\eqref{eq:metrics}), not that the graph becomes
sparser.

\noindent The $\kappa_{\text{med}}$ column converts the dimensionless $\sigma$ into \AA{}: since
$\kappa_i=\sigma m_i/2$, an entry equals $\sigma$ times half the margin median of the sample
($0.0787$\,\AA{} for P69905 and $0.118$\,\AA{} for P62241, Table~\ref{tab:drift}), so $\sigma=1$ permits a
median step of $0.039$\,\AA{} and the five-point sweep runs from $\kappa_{\text{med}}=m_{\text{med}}/4$
($\sigma=1/2$) to $4m_{\text{med}}$ ($\sigma=8$). The default $\sigma=1/2$ therefore confines the median residue
to a quarter of the margin that separates its $k$-th from its $(k{+}1)$-th neighbour; read together with the
$|\Om|$ column, which counts the residues whose displacement reaches the trust radius, this is the quantitative
content of ``$\sigma=1/2$ is too tight'' and of the margin-based calibration advocated above.

\noindent One column is a theorem rather than a measurement. Since
$\Om_t=\{i:\|B^t_i-x^0_i\|_2\ge\kappa_i\}$ with $\kappa_i=\sigma m_i/2$, the risk set is \emph{non-increasing in
$\sigma$ for a fixed basis}: enlarging the radii can only remove indices. Table~\ref{tab:sigma} meets that
condition by construction --- its basis is the
unconstrained input $u$ itself ($\gamma=1$), identical for every $\sigma$ --- so its $|\Om|$ column is monotone
automatically; when the basis is the \emph{projected} quantity, as in the deployed \Level{3}, $B^t$ moves with $\kappa$
and the same trend is empirical only.

\begin{table}[t]
\centering
\caption{Offline oracle replay ($\gamma=1$) on all samples: weighted topology distance $\Dw$ against the
\PB{} edge set at the same boundary, and its change against the frozen basis ($\gamma=0$, $B\equiv x^0$),
$\Delta\Dw=(\Dw-\Dw^{\text{frozen}})/\Dw^{\text{frozen}}$ --- \textbf{negative values denote improvement}, i.e.\ a
shorter distance to \PB{} than the frozen basis. $|\Om|$ is the risk-set size at $t=1$ for every sample,
out of $L$ residues (identically $0$ when $B\equiv x^0$); the five-point $\sigma$ sweep is complete on all
three. The last block is the long sample of Table~\ref{tab:long} (PDB code/UniProt accession).}
\label{tab:sigma}
\footnotesize
\begin{tabular}{@{}llrrrc@{}}
\toprule
Sample & setting & $\kappa_{\text{med}}$ (\AA) & $\Dw$ & $\Delta\Dw$ & $|\Om|$ \\
\midrule
P69905 & frozen ($\gamma{=}0$) & ---   & 0.0578 & ---     & 0 \\
P69905 & $\sigma{=}0.5$        & 0.020 & 0.0515 & $-11\%$ & 133 \\
P69905 & $\sigma{=}1$          & 0.039 & 0.0458 & $-21\%$ & 116 \\
P69905 & $\sigma{=}2$          & 0.079 & 0.0401 & $-31\%$ & 80 \\
P69905 & $\sigma{=}4$          & 0.157 & 0.0274 & $-53\%$ & 47 \\
P69905 & $\sigma{=}8$          & 0.315 & 0.0227 & $-61\%$ & 26 \\
\midrule
P62241 & frozen ($\gamma{=}0$) & ---   & 0.0355 & ---     & 0 \\
P62241 & $\sigma{=}0.5$        & 0.030 & 0.0321 & $-9\%$  & 96 \\
P62241 & $\sigma{=}1$          & 0.059 & 0.0294 & $-17\%$ & 56 \\
P62241 & $\sigma{=}2$          & 0.118 & 0.0287 & $-19\%$ & 40 \\
P62241 & $\sigma{=}4$          & 0.236 & 0.0247 & $-30\%$ & 22 \\
P62241 & $\sigma{=}8$          & 0.473 & 0.0208 & $-41\%$ & 15 \\
\midrule
1AKE/P69441 & frozen ($\gamma{=}0$) & ---   & 0.0282 & ---     & 0 \\
1AKE/P69441 & $\sigma{=}0.5$        & 0.019 & 0.0273 & $-3\%$  & 47 \\
1AKE/P69441 & $\sigma{=}1$          & 0.038 & 0.0262 & $-7\%$  & 28 \\
1AKE/P69441 & $\sigma{=}2$          & 0.075 & 0.0245 & $-13\%$ & 13 \\
1AKE/P69441 & $\sigma{=}4$          & 0.151 & 0.0225 & $-20\%$ & 5 \\
1AKE/P69441 & $\sigma{=}8$          & 0.301 & 0.0189 & $-33\%$ & 3 \\
\bottomrule
\end{tabular}
\end{table}

\begin{table}[t]
\centering
\caption{Long samples recorded on the server with the CUDA backend ($L=214$--$333$, seed 1234, MSA depth 3, three
boundaries), outside the range of Tables~\ref{tab:audit}--\ref{tab:sigma}. $\Dw^{F}$ is the frozen basis and
$\Delta\Dw$ its change at $\sigma=1/2,2,8$ (conventions of Table~\ref{tab:sigma}: negative values denote
improvement); $|\Om|$ is the risk-set size at $t=1$ for $\sigma=1/2$; $\Delta(\text{L}n)$ is $\D$ against the
\PB{} reference of the same family, averaged over the three boundaries.}
\label{tab:long}
\footnotesize
\begin{tabular}{@{}lrrrrrrrrr@{}}
\toprule
Sample & $L$ & $|\Eset|$ & $m_{\text{med}}$ (\AA) & $\Dw^{F}$ & $\Delta\Dw$ ($\tfrac12$/2/8)
       & $|\Om|$ & $\Delta(\text{\Level{1}})$ & $\Delta(\text{\Level{2}})$ & $\Delta(\text{\Level{3}})$ \\
\midrule
1AKE/P69441 & 214 & 13926 & 0.075 & 0.0282 & $-3/-13/-33\%$ & 47 & 0.0046 & 0.0000 & 0.0042 \\
1TIM/P00940 & 248 & 16307 & 0.071 & 0.0320 & $-4/-10/-17\%$ & 58 & 0.0054 & 0.0000 & 0.0051 \\
1CA2/P00918 & 260 & 17330 & 0.080 & 0.0134 & $-28/-36/-61\%$ & 51 & 0.0024 & 0.0000 & 0.0014 \\
1LDM/P00341 & 333 & 21842 & 0.081 & 0.0210 & $-5/-13/-25\%$ & 73 & 0.0033 & 0.0000 & 0.0028 \\
\bottomrule
\end{tabular}
\end{table}

Four samples beyond the calibrated range, recorded on the server with the CUDA backend, test whether the
conclusions survive the change of scale and device (Table~\ref{tab:long}). The part the certificates (\textbf{C2}) control does:
none is degenerate (at most $3$ residues with $m_i=0$, at most $0.03\%$ coincident point pairs), \Level{2} reproduces the
\PB{} edge set exactly on all four ($\D=0.0000$ at every boundary --- the single free Pass-1 iteration of
\textbf{C1} is already enough), and \Level{3} closes $58$--$94\%$ of the \Level{1} gap
($\D(\text{\Level{3}})<\D(\text{\Level{1}})$ in every row). The replay response (\textbf{C3}) is again sample-dependent: the frozen basis is
already close to \PB{} on 1CA2 ($\Dw^{F}=0.0134$) and the sweep still removes $28\%$ at $\sigma=1/2$ and
$61\%$ at $\sigma=8$, whereas on 1TIM ($\Dw^{F}=0.0320$) it removes only $4\%$ and $17\%$; the $\sigma=8$ endpoint
spans from $-17\%$ to $-61\%$ across the four samples. The trend in $\sigma$ is monotone in every row, so the
calibration argument is unchanged, but its payoff is not transferable between samples and has to be measured per
target. 1LDM, the longest sample, exceeded the $100$\,GB host-memory budget and was killed after its three
boundary dumps had been written: its replay rows are complete, but no runtime row is reported for it.

\subsection{What the active-set coverage buys: an ablation}
\label{sec:ablation}
The local repair rests on two properties that are of a different nature. Candidate completeness (the ball of
\eqref{eq:cand} contains the true top-$k$) is \emph{proved}; the coverage of the risk set by the active rows is
\emph{assumed} and, in the deployed rule, obtained from a hop over the frozen edge table. Removing the closure
therefore switches off the second property while leaving the first untouched: the candidate set is still the
certified ball $\cup$ sequence window, only fewer rows are rebuilt. That isolates the coverage condition from the
candidate condition, and the residual $\D$ that remains is a direct measure of what the closure buys.

Tables~\ref{tab:closure-train} and~\ref{tab:closure-l2} report the two radii at which we also ran the training and
the cost comparison; the same ablation extended over five radii, up to the saturated regime in which the clip set
holds only $14$--$18$ residues, is reported in \S\ref{sec:coverage-scale}.

\begin{table}[t]
\centering
\caption{Ablation: closure removed, so the active set is the clip set only. Candidates are unchanged (certified
ball $\cup$ sequence window). Training is nominal: the loss matches the full-rebuild reference of the same sample
(15.34 / 13.18 / 13.53) to within $0.03$, no NaN, and the epoch wall-clock difference stays inside run-to-run
variation on this machine ($\pm4\%$), so no timing claim is made. 1CA2 sits at $95$\,GB of the $100$\,GB host
budget.}
\label{tab:closure-train}
\footnotesize
\begin{tabular}{@{}rrrrr@{}}
\toprule
$\sigma$ & sample & $L$ & loss & maxRSS (GB) \\
\midrule
0.05 & 1AKE & 214 & 15.35 & 61.7 \\
0.05 & 1TIM & 248 & 13.17 & 85.7 \\
0.05 & 1CA2 & 260 & 13.53 & 95.0 \\
0.10 & 1AKE & 214 & 15.37 & 61.7 \\
0.10 & 1TIM & 248 & 13.18 & 85.7 \\
0.10 & 1CA2 & 260 & 13.53 & 95.0 \\
\bottomrule
\end{tabular}
\end{table}

\begin{table}[t]
\centering
\caption{Cost of the same ablation. The active set collapses to the clipped residues, so the rebuilt-row count
drops by an order of magnitude while the per-row candidate count is unchanged by construction.}
\label{tab:closure-l2}
\footnotesize
\begin{tabular}{@{}rrrrrr@{}}
\toprule
$\sigma$ & sample & $\lvert\Om^{\rm clip}\rvert$ ($t{=}1/2/3$) & $\sum_{i}\lvert C_i\rvert$ & $\overline{\lvert C_i\rvert}$ & emit\,/\,$\lvert\Om\rvert L$ \\
\midrule
0.05 & 1AKE & 188/196/200 & 18{,}053--19{,}138 & 96.0 & 44.9\% \\
0.05 & 1TIM & 217/232/239 & 18{,}504--20{,}385 & 85.3 & 34.4\% \\
0.05 & 1CA2 & 222/237/248 & 22{,}791--25{,}350 & 102.5 & 39.4\% \\
0.10 & 1AKE & 147/159/167 & 14{,}217--16{,}147 & 96.6 & 45.2\% \\
0.10 & 1TIM & 169/188/202 & 14{,}467--17{,}360 & 85.5 & 34.5\% \\
0.10 & 1CA2 & 165/185/197 & 17{,}187--20{,}407 & 103.4 & 40.0\% \\
\bottomrule
\end{tabular}
\end{table}

\noindent Four readings.\footnote{The reference full rebuild of the same sample is a deterministic function of the
seed and of the configuration and is independent of $\sigma$ and of the closure, so it is taken from a run of the
same family; that the substitution is legitimate is visible in $\lvert\Eset\rvert$, which reproduces the reference
edge counts exactly (13{,}925 / 16{,}305 / 17{,}328). $\D$ here is the unweighted Jaccard distance; the weighted
quantity $\Dw$ of \S\ref{sec:gap} is not recomputed for this ablation and the two must not be interchanged.}

\begin{enumerate}
\item \textbf{Coverage is the switch, not the candidate set.} With the closure removed the candidate set is
      untouched and the residual is $0.01$--$0.22\%$ of the edges. A provably complete candidate set is therefore
      necessary but not sufficient: what the closure buys is not ``better candidates'' but the rows the candidate
      set is applied to.
\item \textbf{The clip set is a good proxy for the risk set only while $\kappa$ is small.} At
      $\sigma=0.05$--$0.1$ it already holds $147$--$200$ of the $214$--$260$ residues ($57$--$93\%$), so most of
      the risk set is clipped for free; at $\sigma=4$ it holds $14$--$18$, and the closure is what restores
      coverage.
\item \textbf{The two errors move in opposite directions with $\sigma$.} With the active set pinned to the clip
      set, $\D$ \emph{increases} with $\sigma$ (1AKE: $0.09\%\to0.22\%$) because fewer rows are rebuilt, whereas
      the geometric gap $\Dw$ \emph{decreases} with $\sigma$ ($-61\%/-41\%/-33\%$, \S\ref{sec:gap}). A single
      $\sigma$ therefore cannot optimise both: the tracking regime buys fidelity of the basis and pays in
      active-set coverage, and the certified regime does the reverse.
\item \textbf{The price is sample-dependent, not a constant.} 1TIM and 1CA2 pay at most $0.04\%$ because their
      drift is about a quarter of 1AKE's ($\overline{\lvert\Delta d\rvert}=0.015$--$0.018$ against
      $0.066$\,\AA), so whether the closure is needed is a property of the target rather than of the algorithm.
\end{enumerate}

The last point is the one that carries over to the design: for the residual error of the repair, what matters is
\emph{which} rows are rebuilt, not \emph{how} their candidates are drawn --- consistent with the fact that the
candidate side is now exact by construction.

\subsection{How the coverage deficit scales with the trust radius}
\label{sec:coverage-scale}
The ablation of \S\ref{sec:ablation} is a two-point comparison. Here the same operation --- closure removed, ball
candidates unchanged, so the active set is the clip set alone --- is repeated at five radii, from the certified
regime, where most residues are clipped, to the tracking regime, where almost none are. The two smallest radii are
the runs of Table~\ref{tab:closure-train}; the three larger ones were produced in a second session under the same
protocol (1 epoch, MSA depth 8, seed 1234, CUDA). Counters and $\D$ are deterministic functions of the seed and the
configuration and are reproducible across sessions; wall-clock and peak memory are not, and are not compared here.

\begin{table}[t]
\centering
\caption{The ablation over the full trust-radius range. $\lvert\Om^{\rm clip}\rvert$ and $\sum_i\lvert C_i\rvert$
are reported at the first boundary and grow by at most $8\%$ across the three; $\bar\D$ is the mean over $t=1,2,3$
of the unweighted Jaccard distance to a full rebuild of the same sample. The reference edge counts are reproduced
exactly in every row ($\lvert\Eset\rvert=$ 13{,}925 / 16{,}305 / 17{,}328 for 1AKE / 1TIM / 1CA2).}
\label{tab:closure-scale}
\scriptsize
\begin{tabular}{@{}rrrrrr@{}}
\toprule
$\sigma$ & sample & $\lvert\Om^{\rm clip}\rvert$ & $\overline{\lvert C_i\rvert}$ & $\sum_i\lvert C_i\rvert$ & $\bar\D$ \\
\midrule
0.05 & 1AKE & 188 & 96.0  & 18{,}053 & 0.0009 \\
0.05 & 1TIM & 217 & 85.3  & 18{,}504 & 0.0001 \\
0.05 & 1CA2 & 222 & 102.5 & 22{,}791 & 0.0002 \\
0.10 & 1AKE & 147 & 96.6  & 14{,}217 & 0.0022 \\
0.10 & 1TIM & 169 & 85.5  & 14{,}467 & 0.0004 \\
0.10 & 1CA2 & 165 & 103.4 & 17{,}187 & 0.0004 \\
0.25 & 1AKE & 85  & 98.0  & 8{,}411  & 0.0048 \\
0.25 & 1TIM & 98  & 86.0  & 8{,}485  & 0.0012 \\
0.25 & 1CA2 & 98  & 105.0 & 10{,}359 & 0.0016 \\
0.50 & 1AKE & 56  & 100.0 & 5{,}644  & 0.0062 \\
0.50 & 1TIM & 58  & 88.0  & 5{,}161  & 0.0019 \\
0.50 & 1CA2 & 51  & 104.0 & 5{,}325  & 0.0011 \\
4.00 & 1AKE & 14  & 111.0 & 1{,}564  & 0.0084 \\
4.00 & 1TIM & 12  & 94.0  & 1{,}137  & 0.0032 \\
4.00 & 1CA2 & 18  & 108.0 & 1{,}944  & 0.0015 \\
\bottomrule
\end{tabular}
\end{table}

\noindent Four readings.

\begin{enumerate}
\item \textbf{The ablation closes a clean trade-off curve.} Over $\sigma\in[0.05,4]$ the rebuilt-row count falls by
      $13\times$ on 1AKE ($188\to14$) and the candidate work by $11.5\times$ ($18{,}053\to1{,}564$), while the
      residual rises by $9.3\times$ ($0.0009\to0.0084$). The same directions hold on the other two samples
      ($32\times$ and $7.5\times$ in $\bar\D$), the only exception being one inversion of $5\times10^{-4}$ on 1CA2
      between $\sigma=0.25$ and $0.50$.
\item \textbf{The per-row advantage weakens as $\sigma$ grows, the row advantage strengthens.} The mean candidate
      count is not constant: it rises from $96$ to $111$ on 1AKE ($+16\%$) as the shell $\delta^t_i+\delta^t_{\max}$
      widens with the trust radius, so the emit fraction of a full rebuild grows from $44.9\%$ to $52.2\%$. This is
      the measured form of the local-density estimate $\lvert C_i\rvert\approx k\,(r^t_i/\rho_i)^3$: the total still
      falls by $11.5\times$ only because the row count falls faster ($13\times$).
\item \textbf{What the ball buys is compute, not fidelity.} At the three larger radii we also have runs with the
      two-hop candidate set in place. Both constructions produce \emph{bit-identical} graphs --- the same $\D$, the
      same $\lvert\Eset\rvert$ and the same loss --- while the two-hop set costs $1.65$--$2.33\times$ more
      ($1.92$--$2.86\times$ in emit iterations). The candidate side is therefore a pure constant-factor
      improvement at radii where both sets are complete, and the exactness of the repair never depended on it.
\item \textbf{The two errors move in opposite directions with $\sigma$.} Under the pinned active set, $\D$
      increases with $\sigma$, whereas the geometric gap $\Dw$ of \S\ref{sec:gap} decreases
      ($-61\%/-41\%/-33\%$). A single $\sigma$ cannot optimise both: the tracking regime buys fidelity of the basis
      and pays in coverage, the certified regime does the reverse.
\end{enumerate}

We do not report the loss for this sweep. At $\sigma=4$ the ablation departs from the full-rebuild trajectory
(loss $6.33$ against $15.34$ on 1AKE, reproduced in four independent runs), which is itself informative ---
rebuilding fourteen rows is a different graph --- but a single epoch cannot turn that into a statement about
accuracy.

\subsection{Discussion}
\label{sec:discussion}
(i) \emph{Metric proxy}: $w_e=1/(1+\|d_e\|)$ approximates the attention mass; replacing it by the loss
sensitivity $|L(G\cup e)-L(G)|$ would make $\Dw$ directly comparable to training metrics. (ii) \emph{Sample
degeneracy}: one widely used sample maps $24/103$ residues, which makes margins identically zero and inflates
all topology distances; conclusions about margins and absolute geometry are drawn only from the well-mapped
samples. (iii) \emph{Complete-graph identity}: for $L\le65$, $k=\min(64,L-1)$ makes the graph complete, so
$J\equiv1$ is an identity and such samples cannot test stability. (iv) \emph{Trajectory coupling}: runtime
$\D$ is measured against a \PB{} reference of the same family, but each mode integrates its own
trajectory, so $\D$ mixes topology error with trajectory divergence; the offline replay on recorded coordinates
(Table~\ref{tab:sigma}) is the clean instrument. (v) \emph{Accuracy gap unknown}: the existing
\Fix{} vs \PB{} comparison has a seed spread ($\pm1.3$ on a loss of $\approx15.5$) as large as the
effect, so the benefit ceiling of the whole line of work is still unmeasured. (vi) \emph{Sampling of
crossings}: $x(t)$ is recorded only at the $T$ block boundaries, so the crossings counted in
\eqref{eq:crossings} are a \emph{boundary-sampled} statistic --- a margin that is crossed and crossed back
inside one block leaves no trace --- and the measured churn is therefore a \emph{lower bound} on the true flip
rate (the near-tie model of \S\ref{sec:flip} predicts $0.26$--$0.78$ crossings per residue per boundary
against a sampled $\approx0.64$, consistent with this sampling bias).

\section{Conclusion}
We formulated topology-pass refinement as a damped fixed-point problem whose only unprojected fixed point is
\PB, proved exact non-flip certificates with two constants ($m_i/4$ two-sided, $m_i/2$ one-sided) together
with the per-node coordinate-wise box that the implementation projects onto, whose corner fixes the certified
threshold at $\sigma<1/(2\sqrt3)$ (the inscribed ball reaches $\sigma<1/2$ and is analysed in
Appendix~\ref{app:ball}), derived the linear minimum-mean-square-error replay gain together
with a bound on the offset that the projected limit pays when the \PB{} trajectory leaves the trust region
(Lemma~\ref{lem:offset}), and quantified edge flapping through a near-tie spacing law that reproduces the measured
margin \emph{scale} to $1\%$ while leaving the shape of the margin distribution open
($\mathrm{med}/\mathrm{mean}\approx1$, not $\ln2$). The implementation reproduces
\PB{} exactly when the trust region is removed and reduces to the frozen pipeline when compensation is
disabled, giving a continuous and verifiable family of schedules.

The compensation ladder works as a \emph{fidelity} mechanism (\Level{1} within $0.7\%$, \Level{3} interpolating, \Level{2} exact),
and the recorded two-pass trajectory delivers a real $1.76\times$ epoch speed-up with no measurable loss gap.
The ablation of \S\ref{sec:ablation} locates where the residual of the local repair actually comes from: with the
candidate set held fixed and provably complete, removing the active-set coverage alone leaves $0.01$--$0.22\%$ of
the edges wrong, so what remains to be engineered is the choice of rows, not the choice of candidates.
Because the frozen pipeline already matches $99.4\%$ of the \PB{} edge set, the whole compensation family
can only recover the remaining $0.6\%$, so three follow-ups dominate further engineering: (1) the same-seed
\Fix{} vs \PB{} comparison of loss, FAPE and the chi-angle error on well-mapped samples, which decides whether those
edges matter at all; (2) a small-drift regime study (lower per-block cap, later epochs) that characterises how
the risk set contracts during training; and (3) replacing the edge-weight proxy by a loss-sensitivity weight,
which would put $\Dw$ on the same scale as the training objective.

\FloatBarrier   

\clearpage      

\renewcommand{\thesection}{\Alph{section}}
\renewcommand{\thesectiondis}{\Alph{section}.}
\setcounter{section}{0}
\refstepcounter{section}
\section*{Appendix~\thesection: The sphere-projection alternative}
\label{app:ball}
The main text projects onto a product of coordinate-wise boxes. This appendix records the corresponding theory
for the inscribed product of \emph{balls}, i.e.\ the geometry obtained by replacing the clamp with a radial
scaling onto a sphere. It is not used in the main text: it is what the code selects with
\code{PPML\_L3\_PROJ=ball}, and it is the natural comparison point because the Euclidean certificate speaks of
$\Delta=\max_i\|\delta_i\|_2$ rather than of coordinates. Throughout, $\kappa_i=\sigma m_i/2$ as in
\S\ref{sec:blue}.

\subsection{Definition and shape}

\begin{equation}
  \Pi^{\mathrm{ball}}_\kappa(u)
  \;=\; \operatorname*{arg\,min}_{\|y_i-x^0_i\|_2\le\kappa_i\ \forall i}\tfrac12\|y-u\|^2
  \;=\; x^0+\mathrm{rclip}_\kappa\bigl(u-x^0\bigr),
  \qquad
  \mathrm{rclip}_\kappa(w)_i=\frac{w_i}{\max\bigl\{1,\;\|w_i\|_2/\kappa_i\bigr\}} ,
  \label{eq:projball}
\end{equation}
so a node passes through unchanged when $\|u_i-x^0_i\|_2\le\kappa_i$ and is otherwise scaled back along its own
offset onto the sphere $\|y_i-x^0_i\|_2=\kappa_i$. The clipped set is node-wise,
$\Om^{\mathrm{ball}}_t=\{i:\|B^t_i-x^0_i\|_2\ge\kappa_i\}$. Because the projection places a clipped node exactly
on the sphere, ``was this node clipped'' must be read from the \emph{pre-projection} comparison and not recomputed
from $B^t$, where it would be decided by rounding; the implementation stores that flag. The limits agree with the
box: $\gamma=0\Rightarrow B=x^0$; $\gamma=1$ with $\kappa\to\infty\Rightarrow B=x$, the \PB{} anchor; and
$\kappa=0$ collapses every node onto $x^0$.

The two feasible sets are nested,
$\{\|\delta_i\|_2\le\kappa_i\}\subset\{|\delta_{i,c}|\le\kappa_i\}$, the corner of the box lying at
$\sqrt3\,\kappa_i$. At equal half-width the ball is therefore the more conservative set, and the box admits up to
$\sqrt3$ times the Euclidean displacement.

\subsection{Non-expansiveness and the offset bound}

\begin{proposition}[Firm non-expansiveness of the ball projection]\label{prop:fneball}
The map $\Pi^{\mathrm{ball}}_\kappa$ of~\eqref{eq:projball} is firmly non-expansive.
\end{proposition}
\begin{proof}
Its feasible set is a product of non-empty closed balls, hence closed and convex, and $\Pi^{\mathrm{ball}}_\kappa$
is the metric projection onto it, so the argument of Proposition~\ref{prop:fne} applies verbatim; the scalar
inequalities sum because the set is a product.
\end{proof}

Consequently the composition~\eqref{eq:compose} and the contraction modulus $q$ of
Proposition~\ref{prop:contraction} are inherited exactly as for the box, and only the constants change. The offset
bound of Lemma~\ref{lem:offset} becomes
\begin{equation}
  \bigl\|z_\kappa-z_\perp\bigr\|_\infty
  \;\le\;\frac{1}{1-q}\,\sup_{t,i}\Bigl(\bigl\|x^{\perp,t}_i-x^0_i\bigr\|_2-\kappa_i\Bigr)_{+} ,
  \label{eq:offsetball}
\end{equation}
to be compared with the coordinate-wise residual of \eqref{eq:offset}, both taken in the \emph{same} mixed norm
($\ell_2$ per node, $\ell_\infty$ over nodes and time): for the ball the per-node error is radial, so its norm is
the scalar excess above, whereas for the box it is the $\ell_2$ norm of the three per-coordinate excesses. The two
are related by the same $\sqrt3$: for any offset $w$, with the half-width scaled to $\kappa/\sqrt3$,
\[
\|w\|_2-\kappa\;\le\;\sqrt3\Bigl(\max_c|w_c|-\frac{\kappa}{\sqrt3}\Bigr)
\;\le\;\sqrt3\,\Bigl\|\bigl(|w_c|-\frac{\kappa}{\sqrt3}\bigr)_{+}\Bigr\|_2 ,
\]
so a ball of radius $\kappa$ and a box of half-width $\kappa/\sqrt3$ certify the same worst-case displacement,
while at equal $\kappa$ the box is clipped less often and errs less, simply because it permits more.

\subsection{Certification threshold, and what the geometry does not determine}

With $\Delta\le\kappa$ the two-sided condition $4\Delta<m_{\min}$ of Theorem~\ref{thm:two-sided} holds while
\begin{equation}
  \sigma<\frac{1}{2}\ \ (\text{uniform margins}),\qquad
  \sigma<\frac{m_{\min}}{2m_{\max}}\ \ (\text{general}),
  \label{eq:ballsigma}
\end{equation}
to be compared with \eqref{eq:boxsigma} and the rescaling \eqref{eq:boxrescale}: the ball reaches the
certificate's own radius, the box pays a factor $\sqrt3$. Both are sufficient conditions only; exceeding them
leaves a step uncertified, not flipped.

Finally, the geometry does not by itself determine the flip rate. The crossing-rate relation
\eqref{eq:crossings} is a functional of the trajectory's speed and of the near-tie density along it, so a cap on
the length of a single step --- $\kappa$ for the ball, $\sqrt3\,\kappa$ for the box at equal $\sigma$ --- bounds
neither, and a ratio such as
$\mathbb{E}[N]_{\mathrm{ball}}/\mathbb{E}[N]_{\mathrm{box}}\le\sqrt3$ would require an assumption identifying the
two trajectories up to a rescaling of speed and of near-tie density, which we neither make nor can verify.
Empirically the two are compared by replaying the \emph{same} recorded trajectory under both geometries, which is
how the sweeps in \S\ref{sec:gap} are produced.

\end{document}